\documentclass{article}

\newcommand{\arxiv}{}
\ifdefined\arxiv
\usepackage[preprint,nonatbib]{neurips_2021}

\else
\usepackage[final]{neurips_2021}
\fi

\usepackage[utf8]{inputenc} 
\usepackage[T1]{fontenc}    
\usepackage{hyperref}       
\usepackage{url}            
\usepackage{booktabs}       
\usepackage{amsfonts}       
\usepackage{nicefrac}       
\usepackage{microtype}      
\usepackage{xcolor}         

\usepackage{makecell}

\usepackage{hyperref}

\newif\ifspacinglines
\spacinglinestrue
\ifspacinglines
\else
\spacinglinesfalse
\fi

\newif\ifcomm
\commtrue
\ifcomm
\else
\commfalse
\fi

\newcommand{\inlineeqnum}{\refstepcounter{equation}\hfill\mbox{(\theequation)}}

\ifcomm
	\newcommand{\mycomm}[3]{{\footnotesize{{\color{#2} \textbf{[#1: #3]}}}}}
	\newcommand{\CRdel}[1]{\textcolor{red}{\sout{#1}}}
\else
    \newcommand{\mycomm}[3]{}
    \newcommand{\CRdel}[1]{}
\fi

\newcommand{\sign}{\text{sign}}
\newcommand{\Beta}{\mathrm{B}}
\usepackage{graphicx}
\usepackage{subfigure}

\usepackage{bbm}
\usepackage{amsmath}
\usepackage{amssymb}
\usepackage{mathabx}
\usepackage{bm}
\usepackage{mathtools}
\usepackage{amsthm}
\usepackage{xspace}
\usepackage{wasysym}
\usepackage{multirow}
\usepackage{textcomp}
\usepackage{cleveref}
\usepackage{algorithm,algorithmicx,algpseudocode}

\usepackage{multicol}
\raggedcolumns

\usepackage{etoolbox}

\makeatletter

\patchcmd\page@sofar{\kern-\dimen\tw@ \ifdim\dimen\tw@}
  {\kern-\dimen\tw@  \prevdepth\z@ \ifdim\dimen\tw@}
  {\typeout{Success!}}{\ERRORpatching}

\makeatother

\usepackage{thm-restate}
\usepackage{stmaryrd}

\newtheorem{observation}{Observation}
\newtheorem{theorem}{Theorem}
\newtheorem{definition}{Definition}
\newtheorem{lemma}{Lemma}
\newtheorem{corollary}{Corollary}

\newcommand{\T}[1]{\noindent\textbf{#1}}
\newcommand{\RED}{\ensuremath{\mathcal L}\xspace}
\newcommand{\alg}{\mbox{DRIVE}\xspace}
\newcommand{\algnos}{\mbox{DRIVE}}
\newcommand{\algp}{\mbox{DRIVE$^+$}\xspace}

\newcommand{\sx}{S}
\renewcommand{\sc}{S^+}
\newcommand{\matcol}[1]{\ensuremath{C_{#1}}\xspace}
\newcommand{\myeq}[1]{\ensuremath{~=~}}

\title{\alg: One-bit Distributed Mean Estimation}

\newcommand*\samethanks[1][\value{footnote}]{\footnotemark[#1]}
\author{%
Shay Vargaftik \thanks{Equal Contribution.} \ 
\\
VMware Research \\
 \texttt{shayv@vmware.com}
\And
Ran Ben Basat \samethanks[1]\\
University College London\\
\texttt{r.benbasat@cs.ucl.ac.uk}
\And
Amit Portnoy \samethanks[1]\\
Ben-Gurion University \\
\texttt{amitport@post.bgu.ac.il}
\And
Gal Mendelson \\
Stanford University\\
\texttt{galmen@stanford.edu}
\And
Yaniv Ben-Itzhak \\
VMware Research \\
\texttt{ybenitzhak@vmware.com}
\And
Michael Mitzenmacher \\
Harvard University \\
\texttt{michaelm@eecs.harvard.edu}
}

\newcommand{\norm}[1]{\left\lVert#1\right\rVert}
\newcommand{\E}{\mathbb{E}}
\newcommand{\floor}[1]{\left\lfloor#1\right\rfloor}

\newcommand{\parentheses}[1]{\left(#1\right)}
\newcommand{\angles}[1]{\left\langle#1\right\rangle}
\newcommand{\brackets}[1]{\left[#1\right]}
\newcommand{\set}[1]{\left\{#1\right\}}
\newcommand{\abs}[1]{\left|#1\right|}

\newcommand{\sender}{Buffy\xspace}
\newcommand{\receiver}{Angel\xspace}

\newcommand{\sketch}[1]{\mathcal{S}(#1)}
\newcommand{\unsketch}[1]{\mathcal{U}(#1)}
\newcommand{\compress}[1]{\mathcal{C}(#1)}
\newcommand{\decompress}[1]{\mathcal{D}(#1)}

\newcommand{\new}[1]{\textcolor{black}{#1}}

\begin{document}

\maketitle

\begin{abstract}
We consider the problem where $n$ clients transmit $d$-dimensional real-valued vectors using $d(1+o(1))$ bits each, in a manner that allows 
the receiver to approximately reconstruct their mean.
Such compression problems naturally arise in distributed and federated learning.
We provide novel mathematical results and derive computationally efficient algorithms that are more accurate than previous compression techniques.  
We evaluate our methods on a collection of distributed and federated learning tasks, using a variety of datasets, {and show a consistent improvement over the state of the art.}


\end{abstract}

\section{Introduction}

In many computational settings, one wishes to transmit a $d$-dimensional real-valued vector.
For example, in distributed and federated learning scenarios, multiple participants (a.k.a. \emph{clients}) in distributed SGD send gradients to a parameter server that averages them and updates the model parameters accordingly~\cite{mcmahan2017communication}.
In these applications and others (e.g., traditional machine learning methods such K-Means and power iteration~\cite{pmlr-v70-suresh17a} or other methods such as geometric monitoring~\cite{icde2021}), sending \emph{approximations} of vectors may suffice.  Moreover, the vectors' dimension $d$ is often large (e.g., in neural networks, $d$ can exceed a billion~\cite{NIPS2012_6aca9700,shoeybi2019megatron,NEURIPS2019_093f65e0}), so sending compressed vectors is appealing.

Indeed, recent works have studied how to send vector approximations using representations that use a small number of bits per entry (e.g.,~\cite{pmlr-v70-suresh17a,ben2020send,wen2017terngrad,NIPS2017_6c340f25,konevcny2018randomized,caldas2018expanding}). \new{Further, recent work has shown direct training time reduction from compressing the vectors to one bit per coordinate~\cite{bai2021gradient}.}
Most relevant to our work are solutions that address the distributed mean estimation problem. For example, \cite{pmlr-v70-suresh17a} uses the randomized Hadamard transform followed by stochastic quantization (a.k.a. randomized rounding).
When each of the $n$ clients transmits $O(d)$ bits, their Normalized Mean Squared Error (NMSE) is bounded by $O\big(\frac{\log d}{n}\big)$.
They also show a $O(\frac{1}{n})$ bound with $O(d)$ bits via variable-length encoding, albeit at a higher computational cost.
The sampling method of~\cite{konevcny2018randomized} yields an $O\parentheses{\frac{r\cdot R}{n}}$ NMSE bound using $d(1+o(1))$ bits \emph{in expectation}, where $r$ is each coordinate's representation length and $R$ is the normalized average variance of the sent vectors.
Recently, researchers proposed to use Kashin's representation~\cite{caldas2018expanding,lyubarskii2010uncertainty,iaab006}. Broadly speaking, it allows representing a $d$-dimensional vector in (a higher) dimension~$\lambda\cdot d$ for some $\lambda>1$ using small coefficients. This results in an $O\Big({\frac{\lambda^2}{(\sqrt\lambda -1)^4\cdot n}}\Big)$ NMSE bound, \mbox{where each client transmits $\lambda \cdot d(1+o(1))$ bits~\cite{iaab006}.}
\textcolor{black}{A recent work~\cite{davies2021new} suggested an algorithm where if all clients' vectors have pairwise distances of at most $y\in\mathbb R$ (i.e., for any client pair $\mathfrak c_{1},\mathfrak c_{2}$, it holds that $\norm{x_{(\mathfrak c_1)}-x_{(\mathfrak c_2)}}_2\le y$), the resulting MSE is $O(y^2)$ (which is tight with respect to y) using $O(1)$ bits per coordinate on average. This solution provides a stronger MSE bound when vectors are sufficiently close \mbox{(and thus $y$ is small) but does not improve the worst-case guarantee.}}

\ifdefined\arxiv
\vbox{We step back and focus on approximating $d$-dimensional vectors 
using $d(1+o(1))$ bits (e.g., one bit per dimension and a lower order overhead).  
We develop novel biased and unbiased compression techniques based on (uniform as well as structured) random rotations in high-dimensional spheres. Intuitively, after a rotation, the coordinates are identically distributed, allowing us to estimate each coordinate with respect to the resulting distribution. 
Our algorithms do not require expensive operations, such as variable-length encoding or computing the Kashin's representation, and are fast and easy to implement.
We obtain an $O\parentheses{\frac{1}{n}}$ NMSE bound using $d(1+o(1))$ bits, regardless of the coordinates' representation length, improving over previous works. 
Evaluation results indicate that this translates to a consistent \mbox{improvement over the state of the art in different distributed and federated learning tasks.}}
\else
We step back and focus on approximating $d$-dimensional vectors 
using $d(1+o(1))$ bits (e.g., one bit per dimension and a lower order overhead).  
We develop novel biased and unbiased compression techniques based on (uniform as well as structured) random rotations in high-dimensional spheres. Intuitively, after a rotation, the coordinates are identically distributed, allowing us to estimate each coordinate with respect to the resulting distribution. 
Our algorithms do not require expensive operations, such as variable-length encoding or computing the Kashin's representation, and are fast and easy to implement.
We obtain an $O\parentheses{\frac{1}{n}}$ NMSE bound using $d(1+o(1))$ bits, regardless of the coordinates' representation length, improving over previous works. 
Evaluation results indicate that this translates to a consistent \mbox{improvement over the state of the art in different distributed and federated learning tasks.}
\fi

\section{Problem Formulation and Notation}

\T{1b~-~Vector Estimation.} 
We start by formally defining the \emph{1b~-~vector estimation} problem.
A sender, called \sender, gets a real-valued vector $x\in\mathbb R^d$ and sends it using a $d(1+o(1))$ 
bits message (i.e., asymptotically \emph{one bit per coordinate}).
The receiver, called \receiver, uses the message to derive an estimate $\widehat x$ of the original vector $x$.
We are interested in the quantity $\norm{x-\widehat x}_2^2$, which is the sum of squared errors (SSE), and its expected value, the Mean Squared Error (MSE).  For ease of exposition, we hereafter assume that $x\neq 0$ as this special case can be handled with one additional bit.
Our goal is to minimize the \emph{vector}-NMSE (denoted \emph{vNMSE}), defined as the normalized \mbox{MSE, i.e., $\frac{\mathbb E\brackets{\norm{x-\widehat x}_2^2}}{\norm{x}_2^2}$~.}

\T{1b~-~Distributed Mean Estimation.}
The above problem naturally generalizes to the \emph{1b~-~Distributed Mean Estimation} problem. Here, we have a set of $n{\,\in\,}\mathbb N$ \emph{clients} and a \emph{server}. Each client $\mathfrak c\in\set{1,\ldots,n}$ has its own vector $x_{(\mathfrak c)} {\,\in\,}\mathbb R^d$, which it sends using a $d(1{+}o(1))$-bits message to the server. 
The server then produces an estimate $\widehat {x_{\text{avg}}}{\,\in\,}\mathbb R^d$ of the average $x_{\text{avg}}=\frac{1}{n}\sum_{\mathfrak  c=1}^n x_{(\mathfrak c)}$ with the goal of minimizing {its \emph{NMSE}, defined as the average estimate's MSE normalized by the average norm of the clients' original vectors, i.e., $\frac{\E\brackets{\norm{x_{\text{avg}}-\widehat {x_{\text{avg}}}}_2^2}}{\frac{1}{n}\cdot\sum_{\mathfrak  c=1}^n\norm{x_{(\mathfrak  c)}}_2^2}$~.}

\T{Notation.} We use the following notation and definitions throughout the paper:

\textit{Subscripts.} $x_i$ denotes the $i$'th \emph{coordinate} of the vector $x$, to distinguish it from client $\mathfrak c$'s vector $x_{(\mathfrak c)}$. 

\textit{Binary-sign.} For a vector $x\in \mathbb R^d$, we denote its binary-sign function as $\text{sign}(x)$, where $\text{sign}(x)_i=1$ if $x_i\ge 0$ and $\text{sign}(x)_i=-1$ if $x_i < 0$. 

\textit{Unit vector.} For any (non-zero) real-valued vector $x \in \mathbb R^d$, we denote its normalized vector by $\breve{x} = \frac{x}{\norm{x}_2}$. That is, $\breve{x}$ and $x$ has the same direction and it holds that $\norm{\breve{x}}_2=1$.

\textit{Rotation Matrix.} 
A matrix $R\in\mathbb R^{d\times d}$ is a rotation matrix if $R^{T}  R=I$. The set of all rotation matrices is \mbox{denoted as $\mathcal O(d)$.
It follows that $\forall R \in \mathcal{O}(d){:\,}det(R) \in \set{-1,1}$ and $\forall x \in\mathbb R^d{:\,} \norm{x}_2 {=} \norm{Rx}_2$.}

\textit{Random Rotation.}
A random rotation $\mathcal{R}$ is a distribution over all random rotations in $\mathcal O(d)$.
For ease of exposition, we abuse the notation and given $x\in\mathbb R^d$ denote the random rotation of $x$ by $\mathcal R(x)=Rx$, where $R$ is drawn from $\mathcal R$. Similarly, $\mathcal R^{-1}(x)=R^{-1}x=R^Tx$ is the inverse rotation.

\textit{Rotation Property.}
A quantity that determines the guarantees of our algorithms is~${\RED_{\mathcal R, x}^d =\frac{{\norm{\mathcal R(\breve x)}_1^2}}{d}}$ (note the use of the $L_1$ norm).
We show that rotations with high $\RED_{\mathcal R, x}^d$  values yield better estimates.

\T{Shared Randomness.}
We assume that \sender and \receiver have access to shared randomness, e.g., by agreeing on a common PRNG seed. Shared randomness is studied both in communication complexity (e.g., \cite{newman1991private}) and in communication reduction in machine learning systems (e.g., \cite{pmlr-v70-suresh17a,ben2020send}).
In our context, it means that \mbox{\sender and \receiver can generate the same random rotations without communication.}

\section{The \alg Algorithm}

\ifdefined\arxiv
We start by presenting \alg (Deterministically RoundIng randomly rotated VEctors), a novel 1b~-~Vector Estimation algorithm.
First, we show how to minimize \alg's vNMSE. Then, we show how to make \alg unbiased, extending it to the 1b~-~Distributed \mbox{Mean Estimation problem.}

In \alg, \sender uses shared randomness to sample a rotation matrix $R\sim\mathcal R$ and rotates the vector $x\in\mathbb R^d$ by computing $\mathcal R(x)=R x$.
\else
We start by presenting \alg (Deterministically RoundIng randomly rotated VEctors), a novel 1b~-~Vector Estimation algorithm.
Later, we extend \alg to the 1b~-~Distributed Mean Estimation problem.
In \alg, \sender uses shared randomness to sample a rotation matrix $R\sim\mathcal R$ and rotates the vector $x\in\mathbb R^d$ by computing $\mathcal R(x)=R x$.
\fi
\sender then calculates $\sx$, a scalar quantity we explain below. \sender then sends $\big(\sx,\text{sign}(\mathcal R(x))\big)$ to \receiver. 
As we discuss later, sending $\big(\sx,\text{sign}(\mathcal R(x))\big)$ requires $d(1+o(1))$ bits.
In turn, \receiver computes $\widehat {\mathcal R(x)}=\sx\cdot\text{sign}(\mathcal R(x)) \in\set{-\sx,+\sx}^d$.
It then uses the shared randomness to generate the same rotation matrix
and employs the inverse {rotation, i.e., estimates $\widehat x = \mathcal R^{-1}(\widehat {\mathcal R(x)})$.
The pseudocode of \alg appears in Algorithm~\ref{code:alg1}.}

The properties of \alg depend on the rotation $\mathcal R$ and the \emph{scale parameter} $\sx$.
We consider both uniform rotations, that provide stronger guarantees, and structured rotations that are orders of magnitude faster to compute.  
As for the scale $\sx = \sx({x,R})$, its exact formula determines the characteristics of \alg's estimate, e.g., having minimal vNMSE or being unbiased. 
The latter allows us to apply \alg to the 1b~-~Distributed Mean Estimation (Section~\ref{subsec:drive_dme}) and get an NMSE that decreases proportionally to the number of clients.  

\begin{algorithm}[t]
\caption{~\alg}
\label{code:alg1}
\begin{multicols}{2}
\begin{algorithmic}[1]
  \Statex \hspace*{-4mm}\textbf{\sender:}
  \State Compute $\mathcal R(x),\  \sx$.\textcolor{white}{$\big($}
  \State Send $\big(\sx,\text{sign}(\mathcal R(x))\big)$ to \receiver.\textcolor{white}{$\widehat {\mathcal R(x)}$}
\end{algorithmic}
\columnbreak
\begin{algorithmic}[1]
\Statex \hspace*{-4mm}\textbf{\receiver:}
\State Compute $\widehat {\mathcal R(x)}=\sx\cdot\text{sign}\big(\mathcal R(x)\big)$.
\State Estimate $\widehat x = \mathcal R^{-1}\big(\widehat {\mathcal R(x)}\big)$.
\end{algorithmic}
\end{multicols}
\end{algorithm}
We now prove a general result on the SSE of \alg that applies to any random rotation $\mathcal R$ and any vector $x\in\mathbb R^d$. 
In the following sections, we use this result to obtain the vNMSE \mbox{when considering specific rotations and scaling methods as well as analyzing their guarantees.}
\begin{theorem}\label{thm:theoreticalAlg}
The SSE of \alg is: $\norm{x-\widehat x}_2^2=\norm{x}_2^2-2 \cdot \sx \cdot\norm{\mathcal R(x)}_1 + d \cdot \sx^2$~.
\end{theorem}
\begin{proof}
The SSE in estimating $\mathcal R(x)$ using $\widehat {\mathcal R(x)}$ equals that of estimating $x$ using $\widehat x$. Therefore,
\ifdefined\arxiv
\begin{multline}
\norm{x-\widehat x}_2^2 = \norm{\mathcal R(x-\widehat x)}_2^2 = \norm{\mathcal R(x)-\mathcal R(\widehat x)}_2^2 = \norm{{\mathcal R(x)}-\widehat {\mathcal R(x)}}_2^2\\
= \norm{\mathcal R(x)}_2^2  - 2 \angles{\mathcal R(x),\widehat {\mathcal R(x)}} + {\norm{\widehat {\mathcal R(x)}}_2^2} = \norm{x}_2^2-2\angles{\mathcal R(x),\widehat {\mathcal R(x)}} + \norm{\widehat {\mathcal R(x)}}_2^2.\label{eq:thm_main}
\end{multline}
\else
\\
$\hspace*{0.0cm}\norm{x-\widehat x}_2^2 = \norm{\mathcal R(x-\widehat x)}_2^2 = \norm{\mathcal R(x)-\mathcal R(\widehat x)}_2^2 = \norm{{\mathcal R(x)}-\widehat {\mathcal R(x)}}_2^2$\\
$\hspace*{1.0cm}= \norm{\mathcal R(x)}_2^2  - 2 \angles{\mathcal R(x),\widehat {\mathcal R(x)}} + {\norm{\widehat {\mathcal R(x)}}_2^2} = \norm{x}_2^2-2\angles{\mathcal R(x),\widehat {\mathcal R(x)}} + \norm{\widehat {\mathcal R(x)}}_2^2.\inlineeqnum\label{eq:thm_main}$
\fi
Next, we have that,
\ifdefined\arxiv
\begin{align}
    \angles{\mathcal R(x),\widehat {\mathcal R(x)}} &= \sum_{i=1}^d \mathcal R(x)_i\cdot \widehat {\mathcal R(x)}_i =  \sx \cdot\sum_{i=1}^d \mathcal R(x)_i\cdot \text{sign}\big(\mathcal R(x)_i\big) = \sx \cdot\norm{\mathcal R(x)}_1\ ,\label{eq:thm_inner_product_1}\\
\norm{\widehat {\mathcal R(x)}}_2^2 &= \sum_{i=1}^d \widehat {\mathcal R(x)}_i^2 = d\cdot \sx^2~.\label{eq:thm_inner_product_2}
\end{align}
\else

$\angles{\mathcal R(x),\widehat {\mathcal R(x)}} = \sum_{i=1}^d \mathcal R(x)_i\cdot \widehat {\mathcal R(x)}_i =  \sx \cdot\sum_{i=1}^d \mathcal R(x)_i\cdot \text{sign}\big(\mathcal R(x)_i\big) = \sx \cdot\norm{\mathcal R(x)}_1\ ,\inlineeqnum\label{eq:thm_inner_product_1}$\\
$\hspace*{0.85cm}\norm{\widehat {\mathcal R(x)}}_2^2 = \sum_{i=1}^d \widehat {\mathcal R(x)}_i^2 = d\cdot \sx^2~.\inlineeqnum\label{eq:thm_inner_product_2}$
\fi

Substituting Eq.~\eqref{eq:thm_inner_product_1} and Eq.~\eqref{eq:thm_inner_product_2} in Eq.~\eqref{eq:thm_main} yields the result.
\end{proof}
%
%
%
%
%
\section{\alg With a Uniform Random Rotation}
%

We first consider the thoroughly studied uniform random rotation (e.g., ~\cite{mezzadri2006generate,wedderburn1975generating,heiberger1978generation,stewart1980efficient,tanner1982remark}), which we denote by $\mathcal R_U$. The sampled matrix is denoted by $R_U\sim\mathcal R_U$, that is, $\mathcal R_U(x)=R_U\cdot x$.
An appealing property of a uniform random rotation is that, as we show later, it admits a scaling that results in a low constant vNMSE even with unbiased estimates.

\subsection{1b~-~Vector Estimation}

Using Theorem \ref{thm:theoreticalAlg}, we obtain the following result. The result holds for any rotation, including $\mathcal R_U$.
\begin{lemma}\label{cor:biased_alg1_vNMSE}
For any $x\in\mathbb R^d$, \alg's SSE is minimized by $S=\frac{\norm{\mathcal R(x)}_1}{d}$ (that is, $S=\frac{\norm{Rx}_1}{d}$ is determined after $R\sim\mathcal R$ is sampled). This yields a vNMSE of $\frac{\mathbb E\brackets{\norm{x-\widehat x}_2^2}}{\norm{x}_2^2}=1 - \mathbb E\brackets{\RED_{\mathcal R,x}^d}$.
\end{lemma}
\begin{proof}
By Theorem \ref{thm:theoreticalAlg}, to minimize the SSE we require
\ifdefined\arxiv
\begin{align*}
    \hspace*{.4cm}\frac{\partial}{\partial S}\big({\norm{x}_2^2-2 \cdot \sx \cdot\norm{\mathcal R(x)}_1 + d \cdot \sx^2}\big) = -2 \cdot \norm{\mathcal R(x)}_1 + 2 \cdot d \cdot \sx = 0~,
\end{align*}
\else
\vspace*{1mm}\\
$\hspace*{0.9cm}\frac{\partial}{\partial S}\big({\norm{x}_2^2-2 \cdot \sx \cdot\norm{\mathcal R(x)}_1 + d \cdot \sx^2}\big) = -2 \cdot \norm{\mathcal R(x)}_1 + 2 \cdot d \cdot \sx = 0~,$\\\vspace*{1mm}
\fi
leading to $S=\frac{\norm{\mathcal R(x)}_1}{d}$. Then, the SSE of \alg becomes:
\ifdefined\arxiv
{
\begin{multline*}
\norm{x-\widehat x}_2^2 = \norm{x}_2^2-2\cdot\sx \cdot\norm{\mathcal R(x)}_1+d\cdot \sx^2 = 
\norm{x}_2^2-2\cdot\frac{\norm{\mathcal R(x)}_1^2}{d} + d \cdot \frac{\norm{\mathcal R(x)}_1^2}{d^2}
 \\=
\norm{x}_2^2 - \frac{\norm{\mathcal R(x)}_1^2}{d}
=
\norm{x}_2^2 - \frac{\norm{x}_2^2 \cdot \norm{\mathcal R(\breve x)}_1^2}{d}
= \norm{x}_2^2 \big(1-\frac{\norm{\mathcal R(\breve x)}_1^2}{d}\,\big)~.
\end{multline*}
}
\else
\\
$\hspace*{0.9cm}\norm{x-\widehat x}_2^2 = \norm{x}_2^2-2\cdot\sx \cdot\norm{\mathcal R(x)}_1+d\cdot \sx^2 = 
\norm{x}_2^2-2\cdot\frac{\norm{\mathcal R(x)}_1^2}{d} + d \cdot \frac{\norm{\mathcal R(x)}_1^2}{d^2}$\\
$\hspace*{2.33cm}=
\norm{x}_2^2 - \frac{\norm{\mathcal R(x)}_1^2}{d}
=
\norm{x}_2^2 - \frac{\norm{x}_2^2 \cdot \norm{\mathcal R(\breve x)}_1^2}{d}
= \norm{x}_2^2 \big(1-\frac{\norm{\mathcal R(\breve x)}_1^2}{d}\,\big)~.$

\fi
{Thus, the normalized SSE is $\frac{\norm{x-\widehat x}_2^2}{\norm{x}_2^2}=1{-}\RED_{\mathcal R,x}^d$. Taking expectation yields the result. \qedhere}
\end{proof}
%
%
Interestingly, for the uniform random rotation, $\RED_{\mathcal R_U,x}^d$ follows the same distribution for all $x$. 
This is because, by the definition of $\mathcal R_U$, it holds that $\mathcal R_U(\breve x)$ is distributed uniformly over the unit sphere for any $x$.
Therefore \alg's vNMSE depends only on the dimension $d$.
We next analyze the vNMSE attainable by the best possible $\sx$, as given in Lemma~\ref{cor:biased_alg1_vNMSE},  when the algorithm uses $\mathcal R_U$ and is not required to be unbiased. In particular, we state the following theorem whose
\ifdefined\arxiv
 \mbox{proof appears in Appendix \ref{app:biased_drive_vNMSE}. }
\else
 proof appears in Appendix \ref{app:biased_drive_vNMSE} \mbox{(all appendices appear in the Supplementary Material and the extended paper version~\cite{vargaftik2021drive}). }
\fi
\begin{restatable}{theorem}{biaseddrivevNMSE}\label{thm:biased_drive_vNMSE}
For any $x \in \mathbb R^d$, the vNMSE of \alg with $\sx=\frac{\norm{\mathcal R_U(x)}_1}{d}$ is
$\parentheses{1 - \frac{2}{\pi}} \parentheses{ {1-\frac{1}{d}}}$~.
\end{restatable}

\subsection{1b~-~Distributed Mean Estimation}\label{subsec:drive_dme}
An appealing property of \alg with a uniform random rotation, established in this section, is that with a proper scaling parameter $\sx$, the estimate is unbiased. 
That is, for any $x \in \mathbb R^d$, our scale guarantees that $\mathbb E\brackets{\widehat x} = x$.
Unbiasedness is useful when generalizing to the Distributed Mean Estimation problem. 
Intuitively, when $n$ clients send their vectors, any biased algorithm would result in an NMSE that may not decrease with respect to $n$. For example, if they have the same input vector, the bias would remain after averaging. Instead, an unbiased encoding algorithm has the property that when all clients \mbox{act (e.g., use different PRNG seeds) independently, the NMSE decreases proportionally to $\frac{1}{n}$.}

Another useful property of uniform random rotation is that its distribution is unchanged when composed with other rotations. We use it in \mbox{the following theorem's proof, given in Appendix~\ref{app:drive_is_unbiased}.}

\begin{restatable}{theorem}{driveisunbiased}\label{theorem:drive_is_unbised}
For any $x \in \mathbb R^d$, set $\sx = \frac{\norm{x}_2^2}{\norm{\mathcal R_U(x)}_1}$. Then \alg satisfies $\E [\widehat x] = x$.
\end{restatable}

\mbox{Now, we proceed to obtain vNMSE guarantees for \alg's unbiased estimate.}
\begin{lemma}\label{cor:alg1_unbiased_vNMSE}
For any $x\in\mathbb R^d$, \alg with $\sx = \frac{\norm{x}_2^2}{\norm{\mathcal R_U(x)}_1}$ has a vNMSE of 
$\mathbb E\brackets{{\frac{1}{\RED_{\mathcal R_U,x}^d}}} {-} 1$~.
\begin{proof}
By Theorem~\ref{thm:theoreticalAlg}, the SSE of the algorithm satisfies:
\ifdefined\arxiv
{
\begin{equation*}
\begin{aligned}
&\norm{x}_2^2-2\cdot\sx \cdot\norm{R_U \cdot x}_1+d\cdot \sx^2 = 
\norm{x}_2^2-2\cdot  \frac{\norm{x}_2^2}{\norm{R_U \cdot x}_1} \cdot\norm{R_U \cdot x}_1+ d \cdot \parentheses{\frac{\norm{x}_2^2}{\norm{R_U \cdot x}_1}}^2
 \\
&= d \cdot \parentheses{\frac{\norm{x}_2^2}{\norm{x}_2\norm{R_U \cdot \breve x}_1}}^2 {-} \norm{x}_2^2 = d \cdot \frac{\norm{x}_2^2}{\norm{R_U \cdot \breve x}_1^2} {-} \norm{x}_2^2 = \norm{x}_2^2 \cdot \parentheses{\parentheses{\frac{d}{\norm{R_U \cdot \breve x}_1^2}} {-} 1}~.
\end{aligned}    
\end{equation*}
}
\else
\begin{equation*}
\small
\begin{aligned}
&\norm{x}_2^2-2\cdot\sx \cdot\norm{R_U \cdot x}_1+d\cdot \sx^2 = 
\norm{x}_2^2-2\cdot  \frac{\norm{x}_2^2}{\norm{R_U \cdot x}_1} \cdot\norm{R_U \cdot x}_1+ d \cdot \parentheses{\frac{\norm{x}_2^2}{\norm{R_U \cdot x}_1}}^2
 \\
&= d \cdot \parentheses{\frac{\norm{x}_2^2}{\norm{x}_2\norm{R_U \cdot \breve x}_1}}^2 {-} \norm{x}_2^2 = d \cdot \frac{\norm{x}_2^2}{\norm{R_U \cdot \breve x}_1^2} {-} \norm{x}_2^2 = \norm{x}_2^2 \cdot \parentheses{\parentheses{\frac{d}{\norm{R_U \cdot \breve x}_1^2}} {-} 1}~.
\end{aligned}    
\end{equation*}
\fi
\mbox{Normalizing by $\norm{x}_2^2$ and taking expectation over $R_U$ concludes the proof.}
\end{proof}
\end{lemma}
Our goal is to derive an upper bound on the above expression and thus upper-bound the vNMSE. Most importantly, we show that even though the estimate is unbiased and we use only a single bit per coordinate, the vNMSE does not increase with the dimension and is bounded by a small constant. In particular, in Appendix \ref{app:anbiased_drive_large_deviation}, we prove the following:

\begin{restatable}{theorem}{vNMSEofunbiaseddrive}\label{thm:vNMSEofunbiaseddrive}
For any $x \in\ \mathbb R^d$, the vNMSE of \alg with $\sx = \frac{\norm{x}_2^2}{\norm{\mathcal R_U(x)}_1}$
satisfies:\\
\mbox{$(i)$ For all $d \ge 2$, it is at most $2.92$. $(ii)$ For all $d \ge 135$, it is at most $\frac{\pi}{2} - 1 + \sqrt{\frac{{(6\pi^3-12\pi^2)}\cdot\ln d+1}{d}}$.}
\end{restatable}

This theorem yields strong bounds on the vNMSE. For example, the vNMSE is lower than $1$ for $d \ge 4096$ and lower than $0.673$ for $d \ge 10^5$. Finally, we obtain the following corollary,
\begin{corollary}\label{cor:deviation_res_3}
For any $x\in\mathbb R^d$, the vNMSE tends to $\frac{\pi}{2} - 1 \approx 0.571$ as $d\to\infty$~. 
\end{corollary}

Recall that \alg's above scale $\sx$ is a function of both $x$ and the sampled $R_U$. An alternative approach is to \emph{deterministically} set $\sx$ to 
$\frac{\norm{x}_2^2}{\mathbb E\brackets{\norm{\mathcal R_U(x)}_1}}$.
As we prove in Appendix~\ref{app:l1_expected_value}, the resulting scale is $\frac{\norm{x}_2\cdot (d-1)\cdot \Beta(\frac{1}{2},\frac{d-1}{2})}{2d}$, where $\Beta$ is the Beta function. 
Interestingly, this scale no longer depends on $x$ but only on its norm.
In the appendix, we also prove that the resulting vNMSE is bounded by $\frac{\pi}{2} - 1$ \emph{for any $d$}. 
In practice, we find that the benefit is marginal.

Finally, with a vNMSE guarantee for the unbiased estimate by \alg, we obtain the following key result for the
1b~-~Distributed Mean Estimation problem, whose proof appears in Appendix~\ref{app:dme}.
We note that this result guarantees (e.g., see~\cite{beznosikov2020biased}) that distributed SGD, where the participants' gradients are compressed with \alg, converges at the same asymptotic rate as without compression.

\begin{restatable}{theorem}{cordme}\label{cor:dme}
Assume $n$ clients, each with its own vector $x_{(\mathfrak  c)}{\,\in\,}\mathbb R^d$. Let each client independently sample $R_{U,\mathfrak  c}{\,\sim\,}\mathcal R_U$ and set its scale to $\frac{\norm{x_{(\mathfrak c)}}_2^2}{\norm{R_{U,\mathfrak  c}\cdot x_{(\mathfrak  c)}}_1}$. Then, the server average estimate's NMSE satisfies:
$\frac{\E\brackets{\norm{x_{\text{avg}}-\widehat {x_{\text{avg}}}}_2^2}}{\frac{1}{n}{\cdot}\sum_{\mathfrak  c=1}^n\norm{x_{(\mathfrak  c)}}_2^2} {\,=\,} \frac{\mathit{vNMSE}}{n}$, {where \textit{vNMSE} is given by Lemma~\ref{cor:alg1_unbiased_vNMSE} and is bounded by Theorem~\ref{thm:vNMSEofunbiaseddrive}.} 
%
\end{restatable}
To the best of our knowledge, \alg is the first algorithm with a provable NMSE of $O(\frac{1}{n})$ for the 1b~-~Distributed Mean Estimation problem (i.e., with $d(1+o(1))$ bits). 
In practice, we use only $d+O(1)$ bits to implement \alg. We use the $d(1+o(1))$ notation to ensure compatibility with the theoretical results; see Appendix \ref{app:MessageRepresentationLength} for a discussion.


\section{Reducing the vNMSE with \algp}\label{sec:drive_plus}

To reduce the vNMSE further, we introduce the \algp algorithm.
In \algp, we also use a scale parameter, denoted $\sc=\sc(x, R)$ to differentiate it from the scale $\sx$ of \alg.
Here, instead of reconstructing the rotated vector in a symmetric manner, i.e., $\widehat {\mathcal R(x)} \in \sx \cdot \{-1,1\}^d$, we have that $\widehat {\mathcal R(x)} \in \sc \cdot \{ c_1, c_2\}^d$ where  $c_1,c_2$ are computed using K-Means clustering with $K=2$ over the $d$ entries of the rotated vector $\mathcal R(x)$. 
That is, $c_1,c_2$ are chosen to minimize the SSE over any choice of two values. 
This does not increase the (asymptotic) time complexity over the random rotations considered in this paper as solving K-Means for the special case of one-dimensional data is deterministically solvable in $O(d \log d)$ (e.g.,~\cite{gronlund2017fast}).
Notice that \algp still requires $d(1+o(1))$ bits as we communicate $\parentheses{\sc \cdot c_1, \sc \cdot c_2}$ and a single bit per coordinate, indicating its nearest centroid.
We defer the pseudocode and analyses of \algp to Appendix~\ref{app:drive_plus_app}. We show that with proper scaling, for both the 1b~-~Vector Estimation and 1b~-~Distributed Mean Estimation \mbox{problems, \algp yields guarantees that are at least as strong as those of \alg.}


\section{\alg with a Structured Random Rotation}\label{sec:Hadamard}

Uniform random rotation generation usually relies on QR factorization (e.g., see \cite{pytorchqrfact}), which requires $O(d^3)$ time and $O(d^2)$ space.  
Therefore, uniform random rotation can only be used in practice to rotate low-dimensional vectors. This is impractical for neural network architectures with many millions of parameters.
To that end, we continue to analyze \alg and \algp with the (randomized) Hadamard transform, a.k.a. \emph{structured} random rotation~\cite{pmlr-v70-suresh17a,ailon2009fast},  that admits a fast \emph{in-place}, 
\ifdefined\arxiv
\mbox{parallelizable, $O(d\log d)$ time implementation~\cite{fino1976unified,uberHadamard}. We start with a few definitions.}
\else
\mbox{parallelizable, $O(d\log d)$ time implementation~\cite{fino1976unified,uberHadamard,openSource}. We start with a few definitions.}
\fi

\begin{definition}
The Walsh-Hadamard matrix (\cite{horadam2012Hadamard}) $H_{2^k}\in\{+1,-1\}^{2^{k}\times 2^{k}}$ is recursively defined via: \mbox{ 
$ \small
H_{2^k}{=} \begin{pmatrix}
H_{2^{k-1}} & H_{2^{k-1}} \\
H_{2^{k-1}} & -H_{2^{k-1}}
\end{pmatrix} 
$ and $H_1 {=} \begin{pmatrix} 1 \end{pmatrix}$.
Also, $(\frac{1}{\sqrt d}H) \cdot (\frac{1}{\sqrt d}H)^T {=} I$ and $\mathit{det}(\frac{1}{\sqrt d}H) \in [-1,1]$.} 
\end{definition}

\ifdefined\arxiv
\vbox{
\fi
\begin{definition}
Let $R_H$ denote the rotation matrix $\frac{HD}{\sqrt d}\in\mathbb R^{d\times d}$, where $H$ is a Walsh-Hadamard matrix and $D$ is a diagonal matrix whose diagonal entries are i.i.d. Rademacher random variables (i.e., taking values uniformly in $\pm 1$). 
Then $\mathcal R_H(x)=R_H \cdot x =  \frac{1}{\sqrt d} H \cdot (x_1 \cdot D_{11}, \dots, x_d \cdot D_{dd})^T$ is the randomized Hadamard transform of $x$ and $\mathcal R_H^{-1}(x)=R^T_H \cdot x = \frac{DH}{\sqrt d} \cdot x$ is the inverse transform.
\end{definition}
\ifdefined\arxiv
}
\fi
%


\subsection{1b~-~Vector Estimation}

Recall that the vNMSE of \alg, when minimized using $\sx=\frac{\norm{\mathcal R (x)}_1}{d}$, is  $1 - \mathbb E\brackets{\RED_{\mathcal R,x}^d}$ (see Lemma~\ref{cor:biased_alg1_vNMSE}). We now bound this quantity of \alg with a structured random rotation.
\begin{lemma}\label{lem:hadamard_biased}
For any dimension $d\geq 2$ and vector $x \in \mathbb R^d$, the vNMSE of \alg with a structured random rotation and scale $\sx=\frac{\norm{\mathcal R_H (x)}_1}{d}$ is: $1 - \mathbb E\brackets{\RED_{\mathcal R_H,x}^d} \le \frac{1}{2}$.\vspace{-1mm}
\end{lemma}
\begin{proof}
Observe that for all $i$, 
$\mathbb E\brackets{\abs{\mathcal R_H(x)_i}} = \mathbb E\big[\big|{\sum_{j=1}^d \frac{x_j}{\sqrt{d}} H_{ij} D_{jj}}\big|\big]$. 
Since $\set{H_{ij} D_{jj} \mid j\in[d]}$ are i.i.d. Rademacher random variables we can use the Khintchine inequality~\cite{khintchine1923dyadische,szarek1976best} which implies that  $\frac{1}{\sqrt{2d}} \cdot \norm x_2
\le \mathbb E\brackets{\abs{\mathcal R_H(x)_i}} \le \frac{1}{\sqrt{d}} \cdot \norm x_2$
(see~\cite{filmus2012khintchine,latala1994best} for simplified proofs).
We conclude that: 
\ifdefined\arxiv
{
\begin{equation*}
\begin{aligned}
\mathbb E\brackets{\RED_{\mathcal R_H,x}^d} = \frac{1}{d} \cdot \mathbb E\brackets{ \norm{\mathcal R_H(\breve x)}_1^2} \ge \frac{1}{d} \cdot \mathbb E\brackets{ \norm{\mathcal R_H(\breve x)}_1}^2 \ge \frac{1}{d} \cdot \Big(\sum_{i=1}^d \frac{1}{\sqrt{2d}}\Big)^2 = \frac{1}{2}~.
\end{aligned}    
\end{equation*}
}
\else
\\
\hspace*{1.2cm}
$\mathbb E\brackets{\RED_{\mathcal R_H,x}^d} = \frac{1}{d} \cdot \mathbb E\brackets{ \norm{\mathcal R_H(\breve x)}_1^2} \ge \frac{1}{d} \cdot \mathbb E\brackets{ \norm{\mathcal R_H(\breve x)}_1}^2 \ge \frac{1}{d} \cdot (\sum_{i=1}^d \frac{1}{\sqrt{2d}})^2 = \frac{1}{2}~.$\\
\fi
This bound is sharp since for $d\ge 2$ we have that 
$\RED_{\mathcal R_H,x}^d = \frac{1}{2}$ for $x=(\frac{1}{\sqrt 2}, \frac{1}{\sqrt 2},0,\ldots,0)^T$~.
\end{proof}
\vspace{-1mm}
Observe that unlike for the uniform random rotation, $\mathbb{E}\brackets{\RED_{\mathcal R_H,x}^d}$ depends on $x$. We also note that this bound of $\frac{1}{2}$ applies to \algp (with scale $\sc=1$) as we show in Appendix~\ref{app:1bDrive+}.

\subsection{1b~-~Distributed Mean Estimation}\label{sec:dme_hadamard_subsec}

For an arbitrary $x \in \mathbb R^d$ and $\mathcal R$, 
and in particular for $\mathcal{R}_H$, 
the estimates of \alg cannot be made unbiased.  For example,
for $x=(\frac{2}{3}, \frac{1}{3})^T$ we have that $\sign(\mathcal R_H(x))=(D_{11},D_{11})^T$ and thus $\widehat{\mathcal R_H(x)} = \sx\cdot (D_{11},D_{11})^T$. 
This implies that $\widehat x = \mathcal R_H^{-1}(\widehat{\mathcal R_H(x)}) = \frac{1}{\sqrt 2} \cdot D\cdot H\cdot\sx\cdot (D_{11},D_{11})^T=\sqrt 2 \cdot \sx\cdot D\cdot (D_{11}, 0)^T =$ $\sqrt 2 \cdot S\cdot (D_{11}^2,0)^T=(\sqrt 2 \cdot S ,0)^T$.  \mbox{Therefore, $\mathbb E[\widehat x]\neq x$ regardless of the scale.}

Nevertheless, we next provide evidence for why when the input vector is high dimensional and admits finite moments, a structured random rotation performs similarly to a uniform random rotation, yielding all the appealing aforementioned properties. Indeed, it is a common observation that the distribution of machine learning workloads and, in particular, neural network gradients are governed by such distributions (e.g., lognormal \cite{chmiel2020neural} or normal \cite{banner2018post,ye2020accelerating}).

We seek to show that at high dimensions, the distribution of $\mathcal R_H(x)$ is sufficiently similar to that of $\mathcal R_U(x)$. 
By definition, the distribution of $\mathcal R_U(x)$ is that of a uniformly at random distributed point on a sphere. 
Previous studies of this distribution for high dimensions (e.g., \cite{spruill2007asymptotic,diaconis1987dozen,rachev1991approximate,stam1982limit}) have shown that individual coordinates of $\mathcal R_U(x)$ converge to the same normal distribution and that these coordinates are ``weakly'' dependent in the sense that the joint distribution of every $O(1)$-sized subset of coordinates is similar to that of independent normal variables for large $d$. 

We hereafter assume that $x=(x_1, \ldots, x_d)$, where the $x_i$s are i.i.d. and that $\E[x_j^2]=\sigma^2$ and $\E[\abs{x_j}^3]=\rho < \infty$ for all $j$. We show that $\mathcal R_H (x)_i$ converges to the same normal distribution for all $i$.
Let $F_{i,d}(x)$ be the cumulative distribution function (CDF) of $\frac{1}{\sigma} \cdot \mathcal R_H (x)_i$ and $\Phi$ be the CDF of the standard normal \mbox{distribution.
The following lemma, proven in Appendix~\ref{app:proofOfBerryEsseen}, shows the convergence.}
\begin{restatable}{lemma}{proofOfBerryEsseen}\label{lem:proofOfBerryEsseen}
\mbox{For all $i$, $\mathcal R_H (x)_i$ converges to a normal variable: $\sup_{x\in \mathbb R}\abs{F_{i,d}(x)-\Phi(x)} \le \frac{0.409\cdot \rho}{\sigma^3 \sqrt d}$.}
\end{restatable}

\ifdefined\arxiv
\vbox{
\fi
With this result, we continue to lay out evidence for the ``weak dependency'' among the coordinates. We do so by calculating the moments of their joint distribution in increasing subset sizes showing that these moments converge to those of independent normal variables. Previous work has shown that a structured random rotation on vectors with specific distributions results in ``weakly dependent'' normal variables. This line of research~\cite{rader1969new,thomas2013parallel,herendi1997fast} utilized the Hadamard transform for a different purpose. Their goal was to develop a computationally cheap method to generate independent normally distributed variables from simpler (e.g., uniform) distributions. We apply their analysis to our setting.
\ifdefined\arxiv
}
\fi

\mbox{We partially rely on the following observation that the Hadamard matrix satisfies.}

\begin{observation}\label{obs:Hadamard_rows}(\cite{rader1969new})
The Hadamard product (coordinate-wise product), $H_{\angles{i}}\circ H_{\angles{\ell}}$,
of two rows $H_{\angles{i}},H_{\angles{\ell}}$ in the Hadamard matrix yields another row at the matrix $H_{\angles{i}}\circ H_{\angles{\ell}} = H_{\angles{1+(i-1) \oplus (\ell-1)}}$. Here, $(i-1) \oplus (\ell-1)$ is the bitwise {xor of the $(\log d)$-sized binary representation of $(i-1)$ and $(\ell-1)$}. 
It follows that $\sum_{j=1}^d H_{ij}H_{\ell j}=\sum_{j=1}^d (H_{\angles{i}}\circ H_{\angles{\ell}})_{{j}} = \sum_{j=1}^d H_{1+(i-1) \oplus (\ell-1), j}$~.
\end{observation}

We now analyze the moments of the rotated variables, starting with the following observation. It follows from {the sign-symmetry of $D$ and matches the joint distribution of i.i.d. normal variables.}
\begin{observation}
All odd moments containing $\mathcal R_H(x)$ entries are $0$. That is,\\ 
$\hspace*{1.7cm}\forall q\in\mathbb N, \forall {i_1,\ldots,i_{2q+1}}\in\set{1,\ldots,d}:\mathbb E\brackets{\mathcal R_H(x)_{i_1}\cdot \ldots\cdot \mathcal R_H(x)_{i_{2q+1}}}=0$.
\end{observation}

Therefore, we need to examine only even moments. We start with showing that the second moments also match with the distribution of independent normal variables.
\begin{lemma}
For all $i\neq \ell$ it holds that $\E \brackets{(\mathcal R_H \cdot x)_i \cdot (\mathcal R_H \cdot x)_\ell} = 0$, whereas  $\E \brackets{(\mathcal R_H \cdot x)_i^2} = \sigma^2$.
\end{lemma}
\begin{proof}
{Since $\set{D_{jj}\mid j\in\set{1,\ldots,d}}$ are sign-symmetric and i.i.d., }
$
\E \big[(\mathcal R_H \cdot x)_i \cdot (\mathcal R_H \cdot x)_\ell\big] = \E \big[{\frac{1}{d} (\sum_{j=1}^d x_j H_{ij} D_{jj}) \cdot (\sum_{j=1}^d x_j H_{\ell j} D_{jj})}\big] = \E \brackets{x_j^2} \cdot \frac{1}{d} \cdot \sum_{j=1}^d H_{ij} H_{\ell j}$.
Notice that $\sum_{j=1}^d H_{1j}{\,=\,}d$ \mbox{ and $\sum_{j=1}^d H_{ij}=0$ for all $i>1$. Thus, by \Cref{obs:Hadamard_rows} we get $0$ if $i \neq \ell$ and $\sigma^2$ otherwise.\qedhere} 
\end{proof}
We have established that the coordinates are pairwise uncorrelated. Similar but more involved analysis shows that the same trend continues under the assumption of the existence of $x$'s higher moments. In Appendix \ref{app:Hadamard_4th_moments} we analyze the 4th moments showing that they indeed approach the 4th moments of independent normal variables with a rate of $\frac{1}{d}$; the reader is referred to \cite{rader1969new} for further intuition and higher moments analysis.
We therefore expect that using \alg and \algp with Hadamard transform will yield similar results to that of a uniform random rotation at high dimensions and when the input vectors respect the finite moments assumption.

\begin{figure}[]
\centering
\centerline{\includegraphics[width=\textwidth, trim=70 93 70 0, clip]{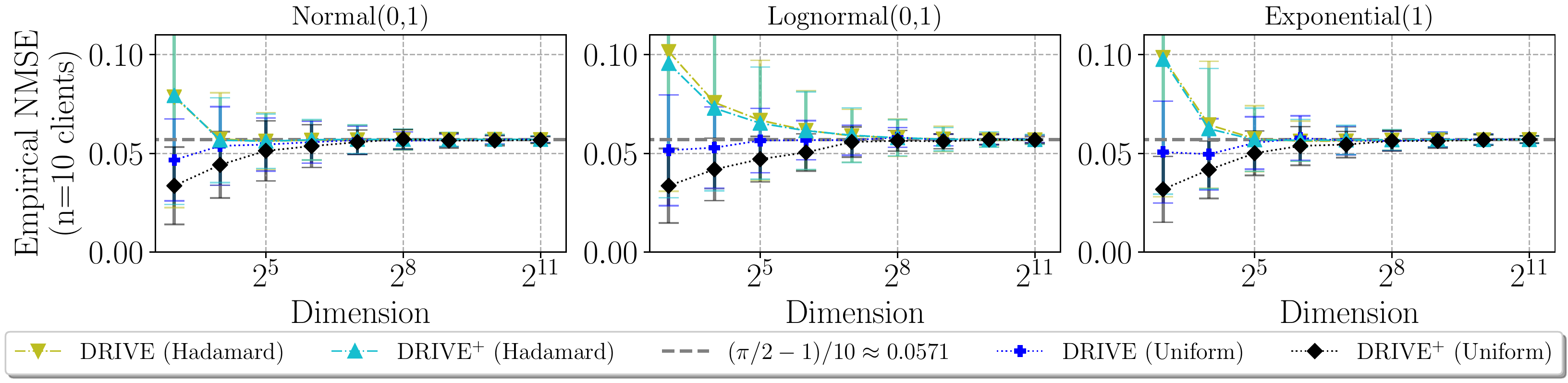}}
\ifdefined\arxiv\else\vspace*{-2mm}\fi
\caption{Distributed mean estimation comparison: each data point is averaged over $10^4$ trials. In each trial, \emph{the same} (randomly sampled) vector is sent by $n=10$ clients.}
\ifdefined\arxiv\else\vspace*{-3mm}\fi
\label{fig:theory_is_cool}
\end{figure}

In addition to the theoretical evidence, in Figure \ref{fig:theory_is_cool}, we show experimental results comparing the measured NMSE for the 1b~-~Distributed Mean Estimation problem with $n=10$ clients (all given the \emph{same} vector so biases do not cancel out) for \alg and \algp using both uniform and structured random rotations over three different distributions. The results indicate that all variants yield similar NMSEs in reasonable dimensions, \mbox{in line with the theoretical guarantee of Corollary \ref{cor:deviation_res_3} and Theorem~\ref{cor:dme}.}

\section{Evaluation}\label{sec:Evaluation}

We evaluate \alg and \algp, comparing them to standard and recent state-of-the-art techniques. We consider classic distributed learning tasks as well as federated learning tasks (e.g., where the data distribution is not i.i.d. and clients may change over time). All the distributed tasks are implemented over PyTorch \cite{NIPS2019_9015} and all the federated tasks are implemented over TensorFlow Federated \cite{tensorflowfed}. 
We focus our comparison on vector quantization algorithms and recent sketching techniques and exclude sparsification methods (e.g.,~\cite{konecy2017federated, WangSLCPW18, NEURIPS2019_d9fbed9d,NEURIPS2018_b440509a}) and methods that involve client-side memory since these can often work in conjunction with our algorithms.

\vbox{
\paragraph{\textbf{{Datasets}.\quad}}\label{p:datasets}
We use MNIST~\cite{lecun1998gradient,lecun2010mnist}, EMNIST~\cite{cohen2017emnist}, CIFAR-10 and CIFAR-100 \cite{krizhevsky2009learning} for image classification tasks; a next-character-prediction task using the Shakespeare dataset \cite{shakespeare}; and a next-word-prediction \mbox{task using the Stack Overflow dataset \cite{stackoverflowdb}. Additional details appear in Appendix~\ref{app:utilizedAssets}.}

\paragraph{\textbf{{Algorithms}.\quad}} 

\addtocounter{footnote}{-1}
\new{Since our focus is on the distributed mean estimation problem and its federated and distributed learning applications, we run \alg and \algp with the unbiased scale quantities.\footnotemark}
}

We compare against several alternative algorithms: (1) \emph{FedAvg} \cite{mcmahan2017communication} that uses the full vectors (i.e., each coordinate is represented using a 32-bit float); (2) Hadamard transform followed by 1-bit stochastic quantization (SQ) \cite{pmlr-v70-suresh17a,konevcny2018randomized}; (3) Kashin’s representation followed by 1-bit stochastic quantization \cite{caldas2018expanding}; (4) \emph{TernGrad}~\cite{wen2017terngrad}, which clips coordinates larger than 2.5 times the standard deviation, then performs 1-bit stochastic quantization on the absolute values and separately sends their signs and the maximum coordinate for scale (we note that TernGrad is a low-bit variant of a well-known algorithm called \emph{QSGD}~\cite{NIPS2017_6c340f25}, and we use TernGrad since we found it to perform better in our experiments);\footnotetext{For \alg the scale is $\sx=\frac{\norm{x}_2^2}{\norm{\mathcal R(x)}_1}$ (see Theorem~\ref{theorem:drive_is_unbised}). For \algp the scale is $\sc=\frac{\norm{x}_2^2}{\norm{c}_2^2}$, where $c\in\set{c_1,c_2}^d$ is the vector indicating the nearest centroid to each coordinate in $\mathcal R(x)$ (see Section~\ref{sec:drive_plus}).}\footnote{\new{When restricted to two quantization levels, TernGrad is identical to QSGD’s max normalization variant with clipping (slightly better due to the ability to represent 0).}} and (5-6)~\emph{Sketched-SGD}~\cite{ivkin2019communication} and \emph{FetchSGD}~\cite{rothchild2020fetchsgd}, which are both count-sketch \cite{charikar2002finding} based algorithms designed for distributed and federated learning, respectively.

    \begin{table}[]
    \resizebox{\textwidth}{!}{%
    \begin{tabular}{r|l|l|l|l|l|l|}
    \cline{2-7}
    \multicolumn{1}{l|}{Dimension ($d$)} & \multicolumn{1}{c|}{\begin{tabular}[c]{@{}c@{}}Hadamard\\ + 1-bit SQ\end{tabular}} & \multicolumn{1}{c|}{\begin{tabular}[c]{@{}c@{}}Kashin\\ + 1-bit SQ\end{tabular}} & \multicolumn{1}{c|}{\begin{tabular}[c]{@{}c@{}}Drive \\ (Uniform)\end{tabular}} & \multicolumn{1}{c|}{\begin{tabular}[c]{@{}c@{}}Drive$^+$\\ (Uniform)\end{tabular}} & \multicolumn{1}{c|}{\begin{tabular}[c]{@{}c@{}}Drive\\  (Hadamard)\end{tabular}} & \multicolumn{1}{c|}{\begin{tabular}[c]{@{}c@{}}Drive$^+$ \\ (Hadamard)\end{tabular}} \\ \hline
    \multicolumn{1}{|r|}{128}            & 0.5308, {\textit{0.34}}     & 0.2550, \textit{2.12}    & 0.0567, \textit{40.4}             & \textbf{0.0547}, \textit{40.7}   &         0.0591,          \textit{0.36}  & 0.0591,  \textit{0.72}    \\ \hline
    \multicolumn{1}{|r|}{8,192}          & 1.3338, {\textit{0.57}}     & 0.3180, \textit{3.42}    & \textbf{0.0571}, \textit{5088}    & \textbf{0.0571}, \textit{5101}   & \textbf{0.0571},        {\textit{0.60}}     & \textbf{0.0571}, \textit{1.06}    \\ \hline
    \multicolumn{1}{|r|}{524,288}        & 2.1456, {\textit{0.79}}     & 0.3178, \textit{4.69}    & ---    & ---   & \textbf{0.0571},         \textit{0.82}      & \textbf{0.0571}, \textit{1.35}    \\ \hline
    \multicolumn{1}{|r|}{33,554,432}     & 2.9332, {\textit{27.1}}     & 0.3179, \textit{332}     & ---    & ---   & \textbf{0.0571},       {\textit{27.2}}    & \textbf{0.0571}, \textit{37.8}     \\ \hline
    \end{tabular}%
    }
    \caption{Empirical NMSE and average per-vector encoding time (in milliseconds, on an RTX 3090 GPU) for distributed mean estimation with $n=10$ clients (same as in Figure~\ref{fig:theory_is_cool}) and Lognormal(0,1) distribution. Each entry is a (NMSE, \textit{time}) tuple and the most accurate result is highlighted in \textbf{bold}.
    }\label{tbl:weAreFast}
    \ifdefined\arxiv\else\vspace*{-2mm}\fi
    \end{table}  

We note that Hadamard with 1-bit stochastic quantization is our most fair comparison, as it uses the same number of bits as \algp (and slightly more than \alg) and has similar computational costs. This contrasts with Kashin's representation, where both the number of bits and the computational costs are higher. 
For example, a standard TensorFlow Federated implementation (e.g., see ``\textsc{class KashinHadamardEncodingStage}'' hyperparameters at \cite{tensorflowfedkashincode}) uses a minimum of $1.17$ bits per coordinate, and three iterations of the algorithm resulting in five Hadamard transforms for each vector.
Also, note that TernGrad uses an extra bit per coordinate for sending the sign. Moreover, \mbox{the clipping performed by TernGrad is a heuristic procedure, which is orthogonal to our work. }

For each task, we use a subset of datasets and the most relevant competition.
Detailed configuration information and additional results appear in Appendix \ref{appendix:additional_simulations}.
We first evaluate the vNMSE-Speed tradeoffs and then proceed to federated and distributed learning experiments.




\paragraph{\textbf{vNMSE-Speed Tradeoff}.\quad}
Appearing in Table~\ref{tbl:weAreFast}, the results show that our algorithms offer the lowest NMSE and that the gap increases with the dimension. 
As expected, \alg and \algp with uniform rotation are more accurate for small dimensions but are significantly slower.
Similarly, \alg is as accurate as \algp, and both are significantly more accurate than Kashin (by a factor of 4.4$\times$-5.5$\times$) and Hadamard (9.3$\times$-51$\times$) with stochastic quantization. 
Additionally, \alg is 5.7$\times$-12$\times$ faster than Kashin and about as fast as Hadamard.
In Appendix~\ref{app:additional_speed_results} we discuss the results
, give the complete experiment specification, and provide measurements on a commodity machine. 

\new{We note that the above techniques, including \alg, are more computationally expensive than linear-time solutions like TernGrad.  Nevertheless, \alg's computational overhead becomes insignificant for modern learning tasks. For example, our measurements suggest that it can take 470~ms for computing the gradient on a ResNet18 architecture (for CIFAR100, batch size = 128, using NVIDIA GeForceGTX 1060 (6GB) GPU) while the encoding of \alg (Hadamard) takes 2.8~ms. That is, the overall computation time is only increased by 0.6\% while the error reduces significantly. Taking the transmission and model update times into consideration would reduce the importance of the compression time further.}

\paragraph{\textbf{{Federated Learning}.\quad}}
We evaluate over four tasks: (1)~EMNIST over customized CNN architecture with two convolutional layers with ${\approx}1.2M$ parameters \cite{caldas2019leaf}; (2)~CIFAR-100 over \mbox{ResNet-18}~\cite{krizhevsky2009learning}
; (3)~a~next-character-prediction task using the Shakespeare dataset \cite{mcmahan2017communication}; (4)~a~next-word-prediction task using the Stack Overflow dataset \cite{reddi2020adaptive}. Both (3) and (4) use LSTM recurrent models \cite{Hochreiter1997LongSM} with
${\approx}820K$ and ${\approx}4M$ parameters,
respectively. We use code, client partitioning, models, hyperparameters, and validation metrics from the federated learning benchmark of \cite{reddi2020adaptive}. 
%

\new{The results are depicted in Figure~\ref{fig:federated_dnn}. 
We observe that in all tasks, \alg and \algp have accuracy that is competitive with that of the baseline, FedAvg. In CIFAR-100, TernGrad and \alg provide the best accuracy. For the other tasks, \alg and \algp have the best accuracy, while the best alternative is either Kashin + 1-bit SQ or TernGrad, depending on the task. Hadamard + 1-bit SQ, which is the most similar to our algorithms (in terms of both bandwidth and compute), provides lower accuracy in all tasks. Additional details and hyperparameter \mbox{configurations are presented in Appendix~\ref{app:fl_details}.}}

\paragraph{\textbf{{Distributed CNN Training}.\quad}}
We evaluate distributed CNN training with 10 clients in two configurations: (1) CIFAR-10 dataset with ResNet-9; (2) CIFAR-100 with ResNet-18 \cite{krizhevsky2009learning, he2016deep}. 
\new{In both tasks, \alg and \algp have similar accuracy to FedAvg, closely followed by Kashin + 1-bit SQ. The other algorithms are less accurate, with Hadamard + 1-bit SQ being better than Sketched-SGD and TernGrad for both tasks.}
Additional details and {hyperparameter configurations are presented in Appendix \ref{app:dl_details}. Figure~\ref{fig:distributed-dnn} depicts the results.}

\begin{figure}[]
\centering
\centerline{\includegraphics[width=\textwidth, trim=10 45 0 0, clip]{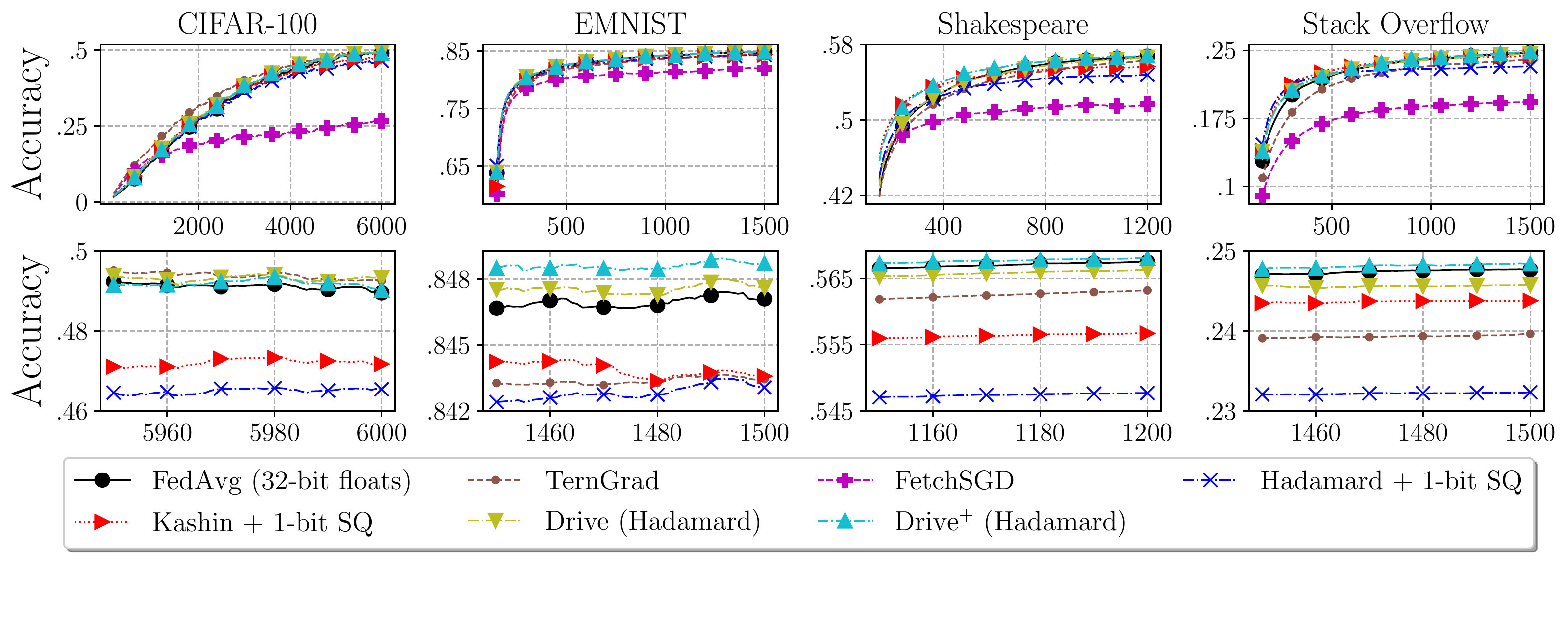}}
\ifdefined\arxiv  
  \vspace*{-2mm}
\fi
\caption{
  Accuracy per round on various federated learning tasks. Smoothing is done using a rolling mean with a window size of 150. The second row zooms-in on the last 50 rounds.}
\label{fig:federated_dnn}
\ifdefined\arxiv  
  \vspace*{-2mm}
\fi
\end{figure}

\begin{figure}[]
\centering
\centerline{\includegraphics[width=\textwidth, trim=0 45 0 0, clip]{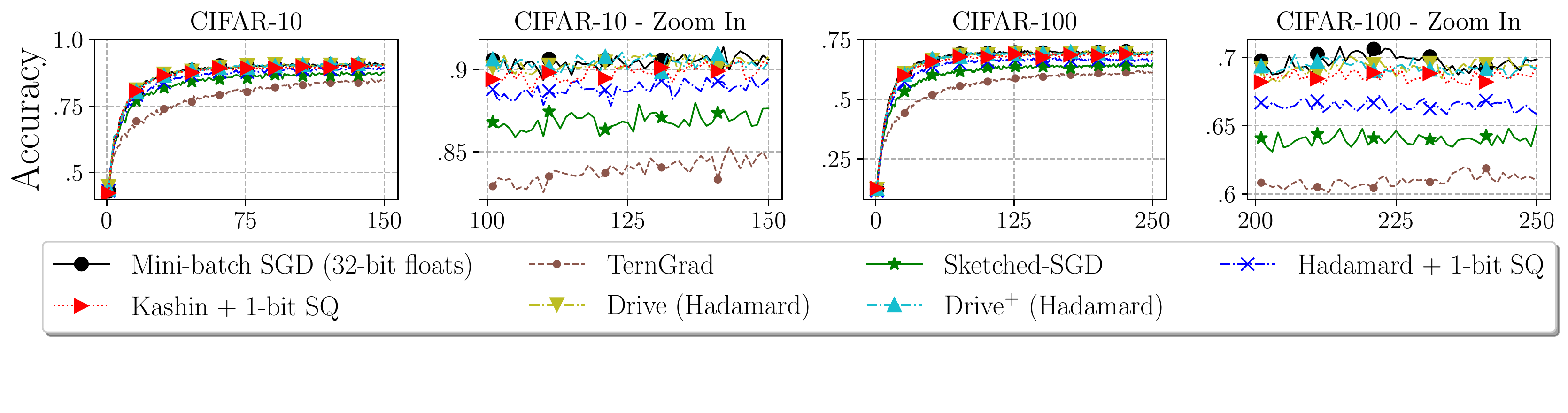}}
\ifdefined\arxiv  
  \vspace*{-2mm}
\fi
\caption{
  \mbox{Accuracy per round on distributed learning tasks, with a zoom-in on the last 50 rounds.}
}
\ifdefined\arxiv  
  \vspace*{-4mm}
\fi
\label{fig:distributed-dnn}
\end{figure}

\paragraph{\textbf{{Evaluation Summary}.\quad}}
Overall, it is evident that \alg and \algp consistently offer markedly favorable results in comparison to the alternatives in our setting. Kashin's representation appears to offer the best competition, albeit at somewhat higher computational complexity and bandwidth requirements. The lesser performance of the sketch-based techniques is attributed to the high noise of the sketch under such a low ($d(1+o(1))$ bits) communication requirement. This is because the \mbox{number of counters they can use is too low, making too many coordinates map into each counter.}
\new{In Appendix~\ref{subsection:power_iteration_appendix}, we also compare \alg and \algp to state of the art techniques over K-Means and Power Iteration tasks for 10, 100, and 1000 clients, yielding similar trends.}


\ifdefined\arxiv
\begin{table}[]
\resizebox{\textwidth}{!}{%
\renewcommand{\arraystretch}{1.6}
\begin{tabular}{c|c|c|c|c|}
\cline{2-5}
\multirow{2}{*}{}                       & \multirow{1}{*}{Scale} & \multicolumn{3}{c|}{Rotation}                     \\ \cline{3-5} Problem
                                        &          $\sx$          & \multicolumn{2}{c|}{Uniform} &   Hadamard                \\ \hline
\multicolumn{1}{|l|}{1b - VE}                  &  $\frac{\norm{\mathcal R(x)}_1}{d}$                  & \multicolumn{2}{c|}{vNMSE $= \parentheses{1 - \frac{2}{\pi}} \parentheses{ {1-\frac{1}{d}}}$} &   vNMSE $\le \frac{1}{2}$                 \\ \hline
\multicolumn{1}{|l|}{\multirow{2}{*}{1b - DME}} & \multirow{2}{*}{$\frac{\norm{x}_2^2}{\norm{\mathcal R(x)}_1}$} & \multicolumn{2}{c|}{$(i)$ $d \ge 2 \implies$ NMSE $\le \frac{1}{n} \cdot 2.92$} & \multirow{2}{*}{---} \\
\multicolumn{1}{|l|}{}                  &                   & \multicolumn{2}{l|}{$(ii)$ 
$d \ge 135 \implies$ NMSE $\le \frac{1}{n} \cdot \left(\frac{\pi}{2} - 1 + \sqrt{\frac{{(6\pi^3-12\pi^2)}\cdot\ln d+1}{d}}\right)$} &                   \\ \hline
\end{tabular}%
}
\caption{\new{Summary of the proven error bounds for \alg.}}
\label{tab:summary_proven_bounds}
\end{table}
\else
\begin{table}[]
\resizebox{\textwidth}{!}{%
\renewcommand{\arraystretch}{1.6}
\begin{tabular}{c|c|c|c|c|}
\cline{2-5}
\multirow{2}{*}{}                       & \multirow{1}{*}{Scale} & \multicolumn{3}{c|}{Rotation}                     \\ \cline{3-5} Problem
                                        &          $\sx$          & \multicolumn{2}{c|}{Uniform} &   Hadamard                \\ \hline
\multicolumn{1}{|l|}{1b - VE}                  &  $\frac{\norm{\mathcal R(x)}_1}{d}$                  & \multicolumn{2}{c|}{vNMSE $= \parentheses{1 - \frac{2}{\pi}} \parentheses{ {1-\frac{1}{d}}}$} &   vNMSE $\le \frac{1}{2}$                 \\ \hline
\multicolumn{1}{|l|}{\multirow{1}{*}{1b - DME}} & \multirow{1}{*}{$\frac{\norm{x}_2^2}{\norm{\mathcal R(x)}_1}$} & \multicolumn{2}{c|}{NMSE $\le \frac{1}{n} \cdot 2.92$;\quad $d \ge 135 \implies$ NMSE $\le \frac{1}{n} \cdot \left(\frac{\pi}{2} - 1 + \sqrt{\frac{{(6\pi^3-12\pi^2)}\cdot\ln d+1}{d}}\right)$} & \multirow{1}{*}{---}\\ \hline
\end{tabular}%
}
\caption{\new{Summary of the proven error bounds for \alg.}}
\label{tab:summary_proven_bounds}
\end{table}
\fi

\section{\new{Discussion}}

\ifdefined\arxiv
In this section, we overview few limitations and future research directions for \alg.
\fi

\paragraph{\textbf{{Proven Error Bounds}.\quad}} We summarize the proven error bounds in Table \ref{tab:summary_proven_bounds}. Since \alg (Hadamard) is generally not unbiased (as discussed in Section \ref{sec:dme_hadamard_subsec}), we cannot establish a formal guarantee for the 1b - DME problem when using Hadamard. It is a challenging research question whether there exists \mbox{other structured rotations with low computational complexity and stronger guarantees.}

\paragraph{\textbf{{Input Distribution Assumption}.\quad}} The distributed mean estimation analysis of our Hadamard-based variants is based on an assumption (Section \ref{sec:dme_hadamard_subsec}) about the vector distributions. While machine learning workloads, and DNN gradients in particular (e.g.,~\cite{chmiel2020neural,banner2018post,ye2020accelerating}), were observed to follow such distributions, this assumption may not hold for other applications.

For such cases, we note that \alg is compatible with the error feedback (EF) mechanism~\cite{seide20141, karimireddy2019error} that ensured convergence and recovery of the convergence rate of non-compressed SGD. Specifically, as evident by Lemma \ref{lem:hadamard_biased}, any scale $\frac{\norm{\mathcal R (x)}_1}{d} \le \sx \le 2 \cdot \frac{\norm{\mathcal R (x)}_1}{d}$ is sufficient to respect the \emph{compressor} (i.e., \emph{bounded variance}) assumption. For completeness, in Appendix \ref{app:exp:ef}, we perform EF experiments comparing \alg and \algp to other compression techniques that use EF.

\paragraph{\textbf{{Varying Communication Budget}.\quad}} 
Unlike some previous works, our algorithms' guarantees with more than one bit per coordinate are not established. It is thus an interesting future work to extend \alg to other communication budgets and understand what are the resulting guarantees. We refer the reader to~\cite{vargaftik2021communication} for initial steps towards that direction.

\paragraph{\textbf{{Entropy Encoding}.\quad}} 

Entropy encoding methods (such as Huffman coding) can further compress vectors of values when the values are not uniformly distributed.  We have compared DRIVE against stochastic quantization methods using entropy encoding for the challenging setting for DRIVE where all vectors are the same (see Table \ref{tbl:weAreFast} for further description).  The results appear in Appendix \ref{subsec:app:ee}, where DRIVE still outperforms these methods.  We also note that, when computation allows and when using \alg with multiple bits per entry, \alg can also be enhanced by entropy encoding techniques.  We describe some initial results for this setting in~\cite{vargaftik2021communication}.


\paragraph{\textbf{{Structured Data}.\quad}} When the data is highly sparse, skewed, or otherwise structured, one can leverage that for compression. We note that some techniques that exploit sparsity or structure can be use in conjunction with our techniques. For example, one may transmit only non-zero entries or Top-K entries \mbox{while compressing these using \alg to reduce communication overhead even further.}

\paragraph{\textbf{{Compatibility With Distributed All-Reduce Techniques}.\quad}} Quantization techniques, including \alg, may introduce overheads in the context of All-Reduce (depending on the network architecture and communication patterns). In particular, if every node in a cluster uses a different rotation, \alg will not allow for efficient in-path aggregation without decoding the vectors. Further, the computational overhead of the receiver increases by a $\log d$ factor as each vector has to be decoded separately before an average can be computed. It is an interesting future direction for \alg to understand how to minimize such potential overheads. For example, one can consider bucketizing co-located workers and apply DRIVE's quantization only for cross-rack traffic.


\section{Conclusions} \label{sec:conclusions}
{
In this paper, we studied the vector and distributed mean estimation problems. These problems are applicable to distributed and federated learning, where clients communicate real-valued vectors (e.g., gradients) to a server for averaging.
To the best of our knowledge, our algorithms are the first with a provable error of $O(\frac{1}{n})$ for the 1b~-~Distributed Mean Estimation problem (i.e., with $d(1+o(1))$ bits). 
As shown in~\cite{iaab006}, any algorithm that uses $O(d)$ shared random bits (e.g., our Hadamard-based variant) has a vNMSE of $\Omega(1)$, i.e., \alg and \algp are asymptotically optimal; additional discussion is given in Appendix~\ref{app:lower_bounds}.
Our experiments, carried over various tasks and datasets, indicate that our algorithms improve over the state of the art. 
All the results presented in this paper are fully reproducible by our source code, 
available at~\cite{openSource}.
}

\ifspacinglines
\fi

\begin{ack}
MM was supported in part by NSF grants CCF-2101140, CCF-2107078, CCF-1563710, and DMS-2023528.  MM and RBB were supported in part by a gift to the Center for Research on Computation and Society at Harvard University. AP was supported in part by the Cyber Security Research Center at Ben-Gurion University of the Negev.
We thank Moshe Gabel, Mahmood Sharif, Yuval Filmus and, the anonymous reviewers for helpful comments and suggestions.
\end{ack}



\bibliography{main}

\begin{thebibliography}{10}

\bibitem{mcmahan2017communication}
H.~Brendan McMahan, Eider Moore, Daniel Ramage, Seth Hampson, and
  Blaise~Ag{\"{u}}era y~Arcas.
\newblock {Communication-Efficient Learning of Deep Networks from Decentralized
  Data}.
\newblock In {\em Artificial Intelligence and Statistics}, pages 1273--1282,
  2017.

\bibitem{pmlr-v70-suresh17a}
Ananda~Theertha Suresh, X~Yu Felix, Sanjiv Kumar, and H~Brendan McMahan.
\newblock {Distributed Mean Estimation With Limited Communication}.
\newblock In {\em International Conference on Machine Learning}, pages
  3329--3337. PMLR, 2017.

\bibitem{icde2021}
Yuval Alfassi, Moshe Gabel, Gal Yehuda, and Daniel Keren.
\newblock A {D}istance-{B}ased {S}cheme for {R}educing {B}andwidth in
  {D}istributed {G}eometric {M}onitoring.
\newblock In {\em 2021 IEEE 37th International Conference on Data Engineering
  ({ICDE})}, 2021.

\bibitem{NIPS2012_6aca9700}
Jeffrey Dean, Greg Corrado, Rajat Monga, Kai Chen, Matthieu Devin, Mark Mao,
  Marc\textquotesingle~aurelio Ranzato, Andrew Senior, Paul Tucker, Ke~Yang,
  Quoc Le, and Andrew Ng.
\newblock {Large Scale Distributed Deep Networks}.
\newblock In F.~Pereira, C.~J.~C. Burges, L.~Bottou, and K.~Q. Weinberger,
  editors, {\em Advances in Neural Information Processing Systems}, volume~25.
  Curran Associates, Inc., 2012.

\bibitem{shoeybi2019megatron}
Mohammad Shoeybi, Mostofa Patwary, Raul Puri, Patrick LeGresley, Jared Casper,
  and Bryan Catanzaro.
\newblock {Megatron-lm: Training multi-billion parameter language models using
  model parallelism}.
\newblock {\em arXiv preprint arXiv:1909.08053}, 2019.

\bibitem{NEURIPS2019_093f65e0}
Yanping Huang, Youlong Cheng, Ankur Bapna, Orhan Firat, Dehao Chen, Mia Chen,
  HyoukJoong Lee, Jiquan Ngiam, Quoc~V Le, Yonghui Wu, and zhifeng Chen.
\newblock {GPipe: Efficient Training of Giant Neural Networks using Pipeline
  Parallelism}.
\newblock In H.~Wallach, H.~Larochelle, A.~Beygelzimer, F.~d\textquotesingle
  Alch\'{e}-Buc, E.~Fox, and R.~Garnett, editors, {\em Advances in Neural
  Information Processing Systems}, volume~32. Curran Associates, Inc., 2019.

\bibitem{ben2020send}
Ran Ben-Basat, Michael Mitzenmacher, and Shay Vargaftik.
\newblock {How to Send a Real Number Using a Single Bit (and Some Shared
  Randomness)}.
\newblock {\em arXiv preprint arXiv:2010.02331}, 2020.

\bibitem{wen2017terngrad}
Wei Wen, Cong Xu, Feng Yan, Chunpeng Wu, Yandan Wang, Yiran Chen, and Hai Li.
\newblock {TernGrad: Ternary Gradients to Reduce Communication in Distributed
  Deep Learning}.
\newblock In {\em Advances in neural information processing systems}, pages
  1509--1519, 2017.

\bibitem{NIPS2017_6c340f25}
Dan Alistarh, Demjan Grubic, Jerry Li, Ryota Tomioka, and Milan Vojnovic.
\newblock {QSGD: Communication-Efficient Sgd via Gradient Quantization and
  Encoding}.
\newblock {\em Advances in Neural Information Processing Systems},
  30:1709--1720, 2017.

\bibitem{konevcny2018randomized}
Jakub Kone{\v{c}}n{\`y} and Peter Richt{\'a}rik.
\newblock {Randomized Distributed Mean Estimation: Accuracy vs. Communication}.
\newblock {\em Frontiers in Applied Mathematics and Statistics}, 4:62, 2018.

\bibitem{caldas2018expanding}
Sebastian Caldas, Jakub Konečný, H~Brendan McMahan, and Ameet Talwalkar.
\newblock {Expanding the Reach of Federated Learning by Reducing Client
  Resource Requirements}.
\newblock {\em arXiv preprint arXiv:1812.07210}, 2018.

\bibitem{bai2021gradient}
Youhui Bai, Cheng Li, Quan Zhou, Jun Yi, Ping Gong, Feng Yan, Ruichuan Chen,
  and Yinlong Xu.
\newblock Gradient compression supercharged high-performance data parallel dnn
  training.
\newblock In {\em The 28th ACM Symposium on Operating Systems Principles (SOSP
  2021)}, 2021.

\bibitem{lyubarskii2010uncertainty}
Yurii Lyubarskii and Roman Vershynin.
\newblock {Uncertainty Principles and Vector Quantization}.
\newblock {\em IEEE Transactions on Information Theory}, 56(7):3491--3501,
  2010.

\bibitem{iaab006}
Mher Safaryan, Egor Shulgin, and Peter Richt{\'a}rik.
\newblock {Uncertainty Principle for Communication Compression in Distributed
  and Federated Learning and the Search for an Optimal Compressor}.
\newblock {\em arXiv preprint arXiv:2002.08958}, 2020.

\bibitem{davies2021new}
Peter Davies, Vijaykrishna Gurunanthan, Niusha Moshrefi, Saleh Ashkboos, and
  Dan Alistarh.
\newblock {New Bounds For Distributed Mean Estimation and Variance Reduction}.
\newblock In {\em International Conference on Learning Representations}, 2021.

\bibitem{newman1991private}
Ilan Newman.
\newblock {Private vs. Common Random Bits in Communication Complexity}.
\newblock {\em Information processing letters}, 39(2):67--71, 1991.

\bibitem{mezzadri2006generate}
Francesco Mezzadri.
\newblock {How to Generate Random Matrices From the Classical Compact Groups}.
\newblock {\em arXiv preprint math-ph/0609050}, 2006.

\bibitem{wedderburn1975generating}
RWM Wedderburn.
\newblock Generating {R}andom {R}otations.
\newblock {\em Research Report}, 1975.

\bibitem{heiberger1978generation}
Richard~M Heiberger.
\newblock {Generation of Random Orthogonal Matrices}.
\newblock {\em Journal of the Royal Statistical Society: Series C (Applied
  Statistics)}, 27(2):199--206, 1978.

\bibitem{stewart1980efficient}
Gilbert~W Stewart.
\newblock {The Efficient Generation of Random Orthogonal Matrices With an
  Application to Condition Estimators}.
\newblock {\em SIAM Journal on Numerical Analysis}, 17(3):403--409, 1980.

\bibitem{tanner1982remark}
Martin~A Tanner and Ronald~A Thisted.
\newblock {Remark as r42: A Remark on as 127. Generation of Random Orthogonal
  Matrices}.
\newblock {\em Journal of the Royal Statistical Society. Series C (Applied
  Statistics)}, 31(2):190--192, 1982.

\bibitem{beznosikov2020biased}
Aleksandr Beznosikov, Samuel Horv{\'a}th, Peter Richt{\'a}rik, and Mher
  Safaryan.
\newblock On biased compression for distributed learning.
\newblock {\em arXiv preprint arXiv:2002.12410}, 2020.

\bibitem{gronlund2017fast}
Allan Gr{\o}nlund, Kasper~Green Larsen, Alexander Mathiasen, Jesper~Sindahl
  Nielsen, Stefan Schneider, and Mingzhou Song.
\newblock {Fast Exact K-Means, K-Medians and Bregman Divergence Clustering in
  1D}.
\newblock {\em arXiv preprint arXiv:1701.07204}, 2017.

\bibitem{pytorchqrfact}
Mario~Lezcano Casado.
\newblock {GeoTorch’s Documentation. A Library for Constrained Optimization
  and Manifold Optimization for Deep Learning in Pytorch.}
\newblock
  \url{https://geotorch.readthedocs.io/en/latest/_modules/geotorch/so.html},
  2021.
\newblock accessed 17-May-21.

\bibitem{ailon2009fast}
Nir Ailon and Bernard Chazelle.
\newblock {The Fast Johnson--Lindenstrauss Transform and Approximate Nearest
  Neighbors}.
\newblock {\em SIAM Journal on computing}, 39(1):302--322, 2009.

\bibitem{fino1976unified}
Bernard~J. Fino and V.~Ralph Algazi.
\newblock Unified matrix treatment of the fast walsh-hadamard transform.
\newblock {\em IEEE Transactions on Computers}, 25(11):1142--1146, 1976.

\bibitem{uberHadamard}
Chunyuan Li, Heerad Farkhoor, Rosanne Liu, and Jason Yosinski.
\newblock Measuring the intrinsic dimension of objective landscapes.
\newblock In {\em International Conference on Learning Representations}, 2018.
\newblock Code available at:
  \url{https://github.com/uber-research/intrinsic-dimension}.

\bibitem{horadam2012Hadamard}
Kathy~J Horadam.
\newblock {\em {Hadamard Matrices and Their Applications}}.
\newblock Princeton university press, 2012.

\bibitem{khintchine1923dyadische}
Aleksandr Khintchine.
\newblock {\"U}ber {D}yadische {B}r{\"u}che.
\newblock {\em Mathematische Zeitschrift}, 18(1):109--116, 1923.

\bibitem{szarek1976best}
S~Szarek.
\newblock {On the Best Constants in the Khinchin Inequality}.
\newblock {\em Studia Mathematica}, 2(58):197--208, 1976.

\bibitem{filmus2012khintchine}
Yuval Filmus.
\newblock {Khintchine-Kahane using Fourier Analysis}.
\newblock {\em Posted at \url{http://www.cs.toronto.edu/yuvalf/KK.pdf}}, 2012.

\bibitem{latala1994best}
Rafa{\l} Lata{\l}a and Krzysztof Oleszkiewicz.
\newblock {On the Best Constant In the Khinchin-Kahane Inequality}.
\newblock {\em Studia Mathematica}, 109(1):101--104, 1994.

\bibitem{chmiel2020neural}
Brian Chmiel, Liad Ben-Uri, Moran Shkolnik, Elad Hoffer, Ron Banner, and Daniel
  Soudry.
\newblock {Neural Gradients Are Near-Lognormal: Improved Quantized and Sparse
  Training}.
\newblock {\em arXiv preprint arXiv:2006.08173}, 2020.

\bibitem{banner2018post}
Ron Banner, Yury Nahshan, Elad Hoffer, and Daniel Soudry.
\newblock {Post-Training 4-Bit Quantization of Convolution Networks for
  Rapid-Deployment}.
\newblock {\em arXiv preprint arXiv:1810.05723}, 2018.

\bibitem{ye2020accelerating}
Xucheng Ye, Pengcheng Dai, Junyu Luo, Xin Guo, Yingjie Qi, Jianlei Yang, and
  Yiran Chen.
\newblock {Accelerating CNN Training by Pruning Activation Gradients}.
\newblock In {\em European Conference on Computer Vision}, pages 322--338.
  Springer, 2020.

\bibitem{spruill2007asymptotic}
Marcus Spruill et~al.
\newblock {Asymptotic Distribution of Coordinates on High Dimensional Spheres}.
\newblock {\em Electronic communications in probability}, 12:234--247, 2007.

\bibitem{diaconis1987dozen}
Persi Diaconis and David Freedman.
\newblock {A Dozen de Finetti-Style Results in Search of a Theory}.
\newblock In {\em Annales de l'IHP Probabilit{\'e}s et statistiques},
  volume~23, pages 397--423, 1987.

\bibitem{rachev1991approximate}
Svetlozar~T Rachev, L~Ruschendorf, et~al.
\newblock {Approximate Independence of Distributions on Spheres and Their
  Stability Properties}.
\newblock {\em The Annals of Probability}, 19(3):1311--1337, 1991.

\bibitem{stam1982limit}
Adriaan~J Stam.
\newblock {Limit Theorems for Uniform Distributions on Spheres in
  High-Dimensional Euclidean Spaces}.
\newblock {\em Journal of Applied probability}, 19(1):221--228, 1982.

\bibitem{rader1969new}
Charles~M Rader.
\newblock {A New Method of Generating Gaussian Random Variables by Computer}.
\newblock Technical report, Massachusetts Institute of Technology, Lincoln Lab,
  1969.

\bibitem{thomas2013parallel}
David~B Thomas.
\newblock {Parallel Generation of Gaussian Random Numbers Using the
  Table-Hadamard Transform}.
\newblock In {\em 2013 IEEE 21st Annual International Symposium on
  Field-Programmable Custom Computing Machines}, pages 161--168. IEEE, 2013.

\bibitem{herendi1997fast}
Tam{\'a}s Herendi, Thomas Siegl, and Robert~F Tichy.
\newblock {Fast Gaussian Random Number Generation Using Linear
  Transformations}.
\newblock {\em Computing}, 59(2):163--181, 1997.

\bibitem{NIPS2019_9015}
Adam Paszke, Sam Gross, Francisco Massa, Adam Lerer, James Bradbury, Gregory
  Chanan, Trevor Killeen, Zeming Lin, Natalia Gimelshein, Luca Antiga, Alban
  Desmaison, Andreas Kopf, Edward Yang, Zachary DeVito, Martin Raison, Alykhan
  Tejani, Sasank Chilamkurthy, Benoit Steiner, Lu~Fang, Junjie Bai, and Soumith
  Chintala.
\newblock {PyTorch: An Imperative Style, High-Performance Deep Learning
  Library}.
\newblock In H.~Wallach, H.~Larochelle, A.~Beygelzimer, F.~d\textquotesingle
  Alch\'{e}-Buc, E.~Fox, and R.~Garnett, editors, {\em Advances in Neural
  Information Processing Systems 32}, pages 8026--8037. Curran Associates,
  Inc., 2019.

\bibitem{tensorflowfed}
Google.
\newblock {TensorFlow Federated: Machine Learning on Decentralized Data}.
\newblock \url{https://www.tensorflow.org/federated}, 2020.
\newblock accessed 25-Mar-20.

\bibitem{konecy2017federated}
Jakub Konečný, H.~Brendan McMahan, Felix~X. Yu, Peter Richtárik,
  Ananda~Theertha Suresh, and Dave Bacon.
\newblock {Federated Learning: Strategies for Improving Communication
  Efficiency}, 2017.

\bibitem{WangSLCPW18}
Hongyi Wang, Scott Sievert, Shengchao Liu, Zachary~B. Charles, Dimitris~S.
  Papailiopoulos, and Stephen Wright.
\newblock {ATOMO: Communication-efficient Learning via Atomic Sparsification}.
\newblock In {\em NeurIPS}, pages 9872--9883, 2018.

\bibitem{NEURIPS2019_d9fbed9d}
Thijs Vogels, Sai~Praneeth Karimireddy, and Martin Jaggi.
\newblock {PowerSGD: Practical Low-Rank Gradient Compression for Distributed
  Optimization}.
\newblock In H.~Wallach, H.~Larochelle, A.~Beygelzimer, F.~d\textquotesingle
  Alch\'{e}-Buc, E.~Fox, and R.~Garnett, editors, {\em Advances in Neural
  Information Processing Systems}, volume~32, 2019.

\bibitem{NEURIPS2018_b440509a}
Sebastian~U Stich, Jean-Baptiste Cordonnier, and Martin Jaggi.
\newblock {Sparsified SGD with Memory}.
\newblock In S.~Bengio, H.~Wallach, H.~Larochelle, K.~Grauman, N.~Cesa-Bianchi,
  and R.~Garnett, editors, {\em Advances in Neural Information Processing
  Systems}, volume~31. Curran Associates, Inc., 2018.

\bibitem{lecun1998gradient}
Yann LeCun, L{\'e}on Bottou, Yoshua Bengio, and Patrick Haffner.
\newblock {Gradient-Based Learning Applied to Document Recognition}.
\newblock {\em Proceedings of the IEEE}, 86(11):2278--2324, 1998.

\bibitem{lecun2010mnist}
Yann LeCun, Corinna Cortes, and CJ~Burges.
\newblock Mnist handwritten digit database.
\newblock {\em ATT Labs [Online]. Available: http://yann.lecun.com/exdb/mnist},
  2, 2010.

\bibitem{cohen2017emnist}
Gregory Cohen, Saeed Afshar, Jonathan Tapson, and Andre Van~Schaik.
\newblock {EMNIST: Extending MNIST to Handwritten Letters}.
\newblock In {\em 2017 International Joint Conference on Neural Networks
  (IJCNN)}, pages 2921--2926. IEEE, 2017.

\bibitem{krizhevsky2009learning}
Alex Krizhevsky, Geoffrey Hinton, et~al.
\newblock {Learning Multiple Layers of Features From Tiny Images}.
\newblock 2009.

\bibitem{shakespeare}
William Shakespeare.
\newblock {The Complete Works of William Shakespeare}.
\newblock \url{https://www.gutenberg.org/ebooks/100}.

\bibitem{stackoverflowdb}
{Stack Overflow Data}.
\newblock \url{https://www.kaggle.com/stackoverflow/stackoverflow}.
\newblock accessed 01-Mar-21.

\bibitem{ivkin2019communication}
Nikita Ivkin, Daniel Rothchild, Enayat Ullah, Vladimir Braverman, Ion Stoica,
  and Raman Arora.
\newblock {Communication-Efficient Distributed Sgd With Sketching}.
\newblock {\em arXiv preprint arXiv:1903.04488}, 2019.

\bibitem{rothchild2020fetchsgd}
Daniel Rothchild, Ashwinee Panda, Enayat Ullah, Nikita Ivkin, Ion Stoica,
  Vladimir Braverman, Joseph Gonzalez, and Raman Arora.
\newblock {FetchSGD: Communication-Efficient Federated Learning With
  Sketching}.
\newblock In {\em International Conference on Machine Learning}, pages
  8253--8265. PMLR, 2020.

\bibitem{charikar2002finding}
Moses Charikar, Kevin Chen, and Martin Farach-Colton.
\newblock Finding frequent items in data streams.
\newblock In {\em International Colloquium on Automata, Languages, and
  Programming}, pages 693--703. Springer, 2002.

\bibitem{tensorflowfedkashincode}
The~TensorFlow Authors.
\newblock {TensorFlow Federated: Compression via Kashin's representation from
  Hadamard transform}.
\newblock
  \url{https://github.com/tensorflow/model-optimization/blob/9193d70f6e7c9f78f7c63336bd68620c4bc6c2ca/tensorflow_model_optimization/python/core/internal/tensor_encoding/stages/research/kashin.py#L92}.
\newblock accessed 16-May-21.

\bibitem{caldas2019leaf}
Sebastian Caldas, Sai Meher~Karthik Duddu, Peter Wu, Tian Li, Jakub Konečný,
  H.~Brendan McMahan, Virginia Smith, and Ameet Talwalkar.
\newblock {LEAF: A Benchmark for Federated Settings}, 2019.

\bibitem{reddi2020adaptive}
Sashank Reddi, Zachary Charles, Manzil Zaheer, Zachary Garrett, Keith Rush,
  Jakub Konečný, Sanjiv Kumar, and H.~Brendan McMahan.
\newblock Adaptive {F}ederated {O}ptimization, 2020.

\bibitem{Hochreiter1997LongSM}
S.~Hochreiter and J.~Schmidhuber.
\newblock {Long Short-Term Memory}.
\newblock {\em Neural Computation}, 9:1735--1780, 1997.

\bibitem{he2016deep}
Kaiming He, Xiangyu Zhang, Shaoqing Ren, and Jian Sun.
\newblock {Deep Residual Learning for Image Recognition}.
\newblock In {\em Proceedings of the IEEE conference on computer vision and
  pattern recognition}, pages 770--778, 2016.

\bibitem{seide20141}
Frank Seide, Hao Fu, Jasha Droppo, Gang Li, and Dong Yu.
\newblock {1-Bit Stochastic Gradient Descent and Its Application to
  Data-Parallel Distributed Training of Speech DNNs}.
\newblock In {\em Fifteenth Annual Conference of the International Speech
  Communication Association}, 2014.

\bibitem{karimireddy2019error}
Sai~Praneeth Karimireddy, Quentin Rebjock, Sebastian Stich, and Martin Jaggi.
\newblock {Error Feedback Fixes SignSGD and other Gradient Compression
  Schemes}.
\newblock In {\em International Conference on Machine Learning}, pages
  3252--3261, 2019.

\bibitem{vargaftik2021communication}
Shay Vargaftik, Ran~Ben Basat, Amit Portnoy, Gal Mendelson, Yaniv Ben-Itzhak,
  and Michael Mitzenmacher.
\newblock Communication-efficient federated learning via robust distributed
  mean estimation.
\newblock {\em arXiv preprint arXiv:2108.08842}, 2021.

\bibitem{openSource}
Shay Vargaftik, Ran~Ben Basat, Amit Portnoy, Gal Mendelson, Yaniv Ben-Itzhak,
  and Michael Mitzenmacher.
\newblock {DRIVE open source code}.
\newblock
  \url{https://github.com/amitport/DRIVE-One-bit-Distributed-Mean-Estimation},
  2021.

\bibitem{muller1959note}
Mervin~E Muller.
\newblock {A Note on a Method for Generating Points Uniformly on N-Dimensional
  Spheres}.
\newblock {\em Communications of the ACM}, 2(4):19--20, 1959.

\bibitem{bennett1962probability}
George Bennett.
\newblock {Probability Inequalities for the Sum of Independent Random
  Variables}.
\newblock {\em Journal of the American Statistical Association},
  57(297):33--45, 1962.

\bibitem{maurer2003bound}
Andreas Maurer et~al.
\newblock {A Bound on the Deviation Probability for Sums of Non-Negative Random
  Variables}.
\newblock {\em J. Inequalities in Pure and Applied Mathematics}, 4(1):15, 2003.

\bibitem{szablowski1998uniform}
Pawe{\l}~J Szab{\l}owski.
\newblock {Uniform Distributions on Spheres in Finite Dimensional $L_\alpha$
  and Their Generalizations}.
\newblock {\em Journal of multivariate analysis}, 64(2):103--117, 1998.

\bibitem{esseen1956moment}
Carl-Gustav Esseen.
\newblock {A Moment Inequality With an Application to the Central Limit
  Theorem}.
\newblock {\em Scandinavian Actuarial Journal}, 1956(2):160--170, 1956.

\bibitem{tensorflowmo}
Google.
\newblock {TensorFlow Model Optimization Toolkit}.
\newblock \url{https://www.tensorflow.org/model_optimization}, 2021.
\newblock accessed 19-Mar-21.

\bibitem{apache2}
Apache license, version 2.0.
\newblock \url{https://www.apache.org/licenses/LICENSE-2.0}.
\newblock accessed 19-Mar-21.

\bibitem{pytorchlicense}
Pytorch license.
\newblock
  \url{https://github.com/pytorch/pytorch/blob/aaccdc39965ade4b61b7852329739e777e244c25/LICENSE}.
\newblock accessed 19-Mar-21.

\bibitem{CCASA3}
{Creative Commons Attribution-Share Alike 3.0 license.}
\newblock \url{https://creativecommons.org/licenses/by-sa/3.0/}.
\newblock accessed 01-Mar-21.

\bibitem{tensorflowfedencsum}
The~TensorFlow Authors.
\newblock {TensorFlow Federated: EncodedSumFactory class}.
\newblock
  \url{https://github.com/tensorflow/federated/blob/v0.19.0/tensorflow_federated/python/aggregators/encoded.py#L40-L142}.
\newblock accessed 19-May-21.

\end{thebibliography}
\ifdefined\arxiv
\bibliographystyle{unsrt}
\else
\bibliographystyle{unsrt}
\fi


 \ifdefined\arxiv
 \else
\section*{Checklist}

\begin{enumerate}

\item For all authors...
\begin{enumerate}
  \item Do the main claims made in the abstract and introduction accurately reflect the paper's contributions and scope?
    \answerYes{}
  \item Did you describe the limitations of your work?
    \answerYes{See Section \ref{sec:conclusions}.}
  \item Did you discuss any potential negative societal impacts of your work?
    \answerNo{We do not believe there is an inherit negative societal impact in this work.}
  \item Have you read the ethics review guidelines and ensured that your paper conforms to them?
    \answerYes{}
\end{enumerate}

\item If you are including theoretical results...
\begin{enumerate}
  \item Did you state the full set of assumptions of all theoretical results?
    \answerYes{}
	\item Did you include complete proofs of all theoretical results?
    \answerYes{}
\end{enumerate}

\item If you ran experiments...
\begin{enumerate}
  \item Did you include the code, data, and instructions needed to reproduce the main experimental results (either in the supplemental material or as a URL)?
    \answerYes{In the Supplemental Material. It will be released as open-source after publication.} 
  \item Did you specify all the training details (e.g., data splits, hyperparameters, how they were chosen)?
    \answerYes{see Section \ref{sec:Evaluation} and Appendix \ref{appendix:additional_simulations}.}
	\item Did you report error bars (e.g., with respect to the random seed after running experiments multiple times)?
    \answerNo{Most experiments can be easily reproduced but require considerable execution time and compute resources. We did include error bars in Figure~\ref{fig:theory_is_cool}, where it was feasible.}
	\item Did you include the total amount of compute and the type of resources used (e.g., type of GPUs, internal cluster, or cloud provider)?
    \answerYes{See Section \ref{sec:Evaluation} and Appendix \ref{appendix:additional_simulations}.}
\end{enumerate}

\item If you are using existing assets (e.g., code, data, models) or curating/releasing new assets...
\begin{enumerate}
  \item If your work uses existing assets, did you cite the creators?
    \answerYes{See Section \ref{sec:Evaluation} and Appendix \ref{appendix:additional_simulations}.} 
  \item Did you mention the license of the assets?
    \answerYes{See Section \ref{sec:Evaluation} and Appendix \ref{appendix:additional_simulations}.} 
  \item Did you include any new assets either in the supplemental material or as a URL?
    \answerYes{We include the source code of our algorithms in the Supplemental Material. It will be released as open-source after publication.}
  \item Did you discuss whether and how consent was obtained from people whose data you're using/curating?
    \answerNA{We used well-known datasets whose creators use either publicly available data or received proper consent.}
  \item Did you discuss whether the data you are using/curating contains personally identifiable information or offensive content?
    \answerNA{We used well-known datasets, which do not contain personally identifiable information or offensive content.}
\end{enumerate}

\item If you used crowdsourcing or conducted research with human subjects...
\begin{enumerate}
  \item Did you include the full text of instructions given to participants and screenshots, if applicable?
    \answerNA{We did not use crowdsourcing or conducted research with human subjects.}
  \item Did you describe any potential participant risks, with links to Institutional Review Board (IRB) approvals, if applicable?
  \answerNA{We did not use crowdsourcing or conducted research with human subjects.}
  \item Did you include the estimated hourly wage paid to participants and the total amount spent on participant compensation?
    \answerNA{We did not use crowdsourcing or conducted research with human subjects.}
\end{enumerate}

\end{enumerate}

 \fi



\appendix

\newpage

\section{Missing Proofs for \alg}

\subsection{Proof of Theorem~\ref{thm:biased_drive_vNMSE}}\label{app:biased_drive_vNMSE}
\biaseddrivevNMSE*
\begin{proof}
Let $T \in \mathcal S^{d-1}$ be uniformly distributed on the unit sphere.
By the definition of $\mathcal R_U$, for \emph{any} $x\in\mathbb R^d$, $\mathcal R_U(\breve x)$ and $T$ follow the same distribution (i.e., $\mathcal R_U(\breve x)$ is also uniformly distributed on the unit sphere). As was established in~\cite{muller1959note}, a uniformly distributed point on a sphere $T \in \mathcal S^{d-1}$ can be obtained by first deriving a random vector $Z = (Z_1,\ldots,Z_d)$ where $Z_i\sim N(0,1), i\in[1,\ldots,d]$ are normally distributed i.i.d. random variables, and then normalizing its norm by $T = \frac{Z}{\norm{Z}_2}$. We also make use of the fact that the norm and direction of a standard multivariate normal vector are independent. That is, $T$ and $\norm{Z}_2$ are independent. Therefore, we have that, 
$\norm{T}_1^2 = \norm{\frac{Z}{\norm{Z}_2}}_1^2 = \frac{1}{\norm{Z}_2^2} \cdot \norm{Z}_1^2 \text{ and thus } \norm{T}_1^2 \cdot \norm{Z}_2^2 = \norm{Z}_1^2.$
Taking expectation and rearranging yields,
{
\begin{equation*}
\begin{aligned}
& \E \brackets{\norm{T}_1^2} = \frac{\E \brackets{\norm{Z}_1^2}}{\E \brackets{\norm{Z}_2^2}} = \frac{\E \brackets{\sum_{i=1}^d Z_i^2 + \sum_{(i,j) \in \set{d \times d}, i \neq j} \abs{Z_i} \cdot \abs{Z_j}}}{d} \\
&= \frac{\sum_{i=1}^d\E \brackets{Z_i^2} + \sum_{(i,j) \in \set{d \times d}, i \neq j} \E \brackets{\abs{Z_i}} \cdot \E \brackets{\abs{Z_j}}}{d} = \frac{d + d\cdot(d{-}1)\cdot(\sqrt{\frac{2}{\pi}})^2}{d} = 1 + (d{-}1)\frac{2}{\pi}~.
\end{aligned}
\end{equation*}
}
Here we employed $\mathbb E \brackets{Z_i^2} = 1$ and that $\abs{Z_i}$ and  $\abs{Z_j}$ are independent half-normal random variables for $i \neq j$ with $\mathbb E \brackets{\abs{Z_i}} =\mathbb E \brackets{\abs{Z_j}} = \sqrt \frac{2}{\pi}$.
Finally, the resulting vNMSE depends only on the dimension and not the specific vector $x$. In particular, $1 - \mathbb E\brackets{\RED_{\mathcal R,x}^d} {=} 1 - \frac{1 + (d-1)\frac{2}{\pi}}{d} {=} \parentheses{1 - \frac{2}{\pi}} \parentheses{ {1-\frac{1}{d}}}$~.
\end{proof}

\subsection{Proof of Theorem~\ref{theorem:drive_is_unbised}}\label{app:drive_is_unbiased}
\driveisunbiased*
\begin{proof}
For any $x \in \mathbb R^d$ denote $x'=(\norm x_2,0,\ldots,0)^T$
and let $R_{x\shortrightarrow x'}\in\mathbb R^{d\times d}$ be a rotation matrix such that $R_{x\shortrightarrow x'}\cdot x = x'$.
Further, denote $R_{x} = R_U R_{x \shortrightarrow x'}^{-1}$.   

Using these definitions and observing that $\norm{\mathcal R_U(x)}_1 = \angles{R_U \cdot x,~\sign(R_U \cdot x)}$ we obtain:
\begin{equation}
\begin{aligned}\label{eq:alg_est_x_hat_unbiased}
    \widehat x \myeq{1} & R_{x\shortrightarrow x'}^{-1} \cdot R_{x\shortrightarrow x'}\cdot\widehat x
    \myeq{2}
    R_{x\shortrightarrow x'}^{-1} \cdot R_{x\shortrightarrow x'} \cdot  R_U^{-1}\cdot\sx\cdot \text{sign}\parentheses{R_U\cdot x}   \\
    \myeq{3}&R_{x\shortrightarrow x'}^{-1} \cdot R_x^{-1}\cdot{\sx\cdot \text{sign}\parentheses{R_x\cdot R_{x\shortrightarrow x'}\cdot x}} 
    \myeq{4} \sx\cdot R_{x\shortrightarrow x'}^{-1} \cdot R_x^{-1}\cdot{\text{sign}\parentheses{R_x\cdot x'}} \\
    \myeq{5}& \frac{\norm{x}_2^2}{\angles{R_U \cdot x,~\sign(R_U \cdot x)}}\cdot  R_{x\shortrightarrow x'}^{-1}\cdot R_x^{-1}\cdot{ \text{sign}\parentheses{R_x\cdot x'}} \\ 
    \myeq{6}& R_{x\shortrightarrow x'}^{-1}\cdot \norm{x}_2^2 \cdot \frac{R_x^{-1}\cdot{ \text{sign}\parentheses{R_x\cdot x'}}}{\angles{R_x \cdot x',~\sign(R_x \cdot x')}}~. 
\end{aligned}    
\end{equation}

Next, let $\matcol{i}$ be a vector containing the values of the $i$'th column of $R_x$. Then, $R_x \cdot x'=\norm x_2 \cdot \matcol{0}$ and $\sign(R_x \cdot x') = \sign(\norm x_2 \cdot \matcol{0}) = \sign(\matcol{0})$. This means that, 
\begin{equation}\label{eq:inner_product_to_norms}
\angles{R_x \cdot x',~\sign(R_x \cdot x')} = \norm{x}_2 \cdot \angles{\matcol{0},~\sign(\matcol{0})} = \norm{x}_2 \cdot \norm{\matcol{0}}_1.   
\end{equation}
Now, since $R_x^{-1}=R_x^T$, we have that, 
\begin{equation}\label{eq:x'est_new}
\begin{aligned}
    R_x^{-1}\cdot{ \sign\parentheses{R_x\cdot x'}} =  \parentheses{\norm{\matcol{0}}_1,\angles{\matcol{1},\sign(\matcol{0})},\ldots,\angles{\matcol{d-1},\sign(\matcol{0})}}^T. 
\end{aligned}   
\end{equation} 
Using Eq. \eqref{eq:inner_product_to_norms} and \eqref{eq:x'est_new} in Eq. \eqref{eq:alg_est_x_hat_unbiased} yields 
\begin{equation}\label{eq:drive_unbiased_before_rotating_back}
\begin{aligned}
\widehat x = R_{x\shortrightarrow x'}^{-1} \cdot \norm{x}_2 \cdot \parentheses{1,\frac{\angles{\matcol{1},\sign(\matcol{0})}}{\norm{\matcol{0}}_1},\ldots,\frac{\angles{\matcol{d-1},\sign(\matcol{0})}}{\norm{\matcol{0}}_1}}^T.
\end{aligned}   
\end{equation} 


Next, by drawing insight from \eqref{eq:drive_unbiased_before_rotating_back}, we show that the estimate of \alg is unbiased by constructing a symmetric algorithm to \alg whose reconstruction's expected value is identical to that of \alg, but with the additional property that the average of both reconstructions is exactly $x$ for each specific run of the algorithm.

In particular, consider an algorithm \algnos$'$ that operates exactly as \alg but, instead of directly using the sampled rotation matrix $R_U = R_x \cdot R_{x \shortrightarrow x'}^{-1}$ it calculates and uses the rotation matrix $R_U' = R_x \cdot I' \cdot R_{x \shortrightarrow x'}^{-1}$ where $I'$ is identical to the $d$-dimensional identity matrix with the exception that $I'_{00}=-1$ instead of $1$. 

Now, since $R_U \sim \mathcal R_U$ and $R_{x\shortrightarrow x'}$ is a fixed rotation matrix, we have that $R_{x} = R_U \cdot R_{x\shortrightarrow x'} \sim \mathcal R_U$. In turn, this also means that $R_{x} \cdot I' \sim \mathcal R_U$ since $I'$ is a fixed rotation matrix.

Consider a run of both algorithm where $\widehat x$ is the reconstruction of \alg for $x$ with a sampled rotation $R_U$ and $\widehat x'$ is the corresponding reconstruction of \algnos$'$ for $x$ with the rotation $R_U'$. 

According to \eqref{eq:drive_unbiased_before_rotating_back} it holds that:
$\widehat x + \widehat x' =R_{x\shortrightarrow x'}^{-1} \cdot \norm{x}_2 \cdot \parentheses{2,0,\ldots,0}^T = 2 \cdot x$. This is because both runs are identical except that the first column of $R_{x}$ (i.e., $\matcol{0}$) and $R_{x} \cdot I'$ have opposite signs. In particular, it holds that $\E \brackets{\widehat x + \widehat x'} = 2 \cdot x$. But, since $R_{x} \sim \mathcal R_U$ and $R_{x} \cdot I' \sim \mathcal R_U$, both algorithms have the same expected value. This yields  $\E \brackets{\widehat x} = \E \brackets{\widehat x'} = x$. This concludes the proof.
\end{proof}

\subsection{Proof of Theorem \ref{thm:vNMSEofunbiaseddrive}}\label{app:anbiased_drive_large_deviation}

The proof of the theorem follows from the following two lemmas. Lemma \ref{thm:anbiased_drive_medium_deviation} relies on Chebyshev's inequality and allows us to bound the vNMSE for all $d$. Lemma \ref{thm:anbiased_drive_large_deviation} gives a sharper bound for large dimensions using the Bernstein's inequality.

\begin{restatable}{lemma}{thmanbiaseddrivemediumdeviation}\label{thm:anbiased_drive_medium_deviation}
 For any constant $k$ and dimension $d \ge 2$ such that $0 < k < \frac{d}{d-1} \cdot \sqrt{\frac{\pi}{\pi - 3}} \cdot \frac{2}{\Beta(\frac{1}{2},\frac{d-1}{2})}$, 
\begin{align*}
    \E \brackets{\frac{d}{\norm{\mathcal R(\breve x)}_1^2}} \le \frac{d}{k^2} + \frac{\parentheses{1-\frac{1}{k^2}}}{\parentheses{- k \cdot  \sqrt{\frac{\pi-3}{\pi d}} + \frac{2\cdot \sqrt d}{(d-1)\cdot \Beta(\frac{1}{2},\frac{d-1}{2})} }^{2}}~.
\end{align*}
\end{restatable}
\begin{proof}
By Lemma \ref{lem:l1_expected_value} in Appendix \ref{app:l1_expected_value} we have that $\E\brackets{\norm{R_U \cdot \breve x}_1} = \frac{2d}{(d-1)\cdot \Beta(\frac{1}{2},\frac{d-1}{2})}$. Thus, 
\begin{align*}
  \mbox{Var}\parentheses{\frac{\norm{R_U \cdot \breve x}_1}{\sqrt d}} &= \E\brackets{\frac{\norm{R_U \cdot \breve x}_1^2}{d}} - \parentheses{\E \brackets{\frac{\norm{R_U \cdot \breve x}_1}{\sqrt d}}}^2 \\&= \frac{2}{\pi} + \frac{1-\frac{2}{\pi}}{d} - 
  \parentheses{\frac{2\cdot \sqrt d}{(d-1)\cdot \Beta(\frac{1}{2},\frac{d-1}{2})}}^2
  \le \frac{\pi-3}{\pi d}. 
\end{align*}
According to Chebyshev's inequality,
$\Pr\brackets{ - \frac{\norm{R_U \cdot \breve x}_1}{\sqrt d} + \frac{2\cdot \sqrt d}{(d-1)\cdot \Beta(\frac{1}{2},\frac{d-1}{2})} \ge k \cdot \sqrt{\frac{\pi-3}{\pi d}}} \le \frac{1}{k^2}$ which is equivalent to 
$\Pr\brackets{\frac{\norm{R_U \cdot \breve x}_1}{\sqrt d} \le - k \cdot  \sqrt{\frac{\pi-3}{\pi d}} + \frac{2\cdot \sqrt d}{(d-1)\cdot \Beta(\frac{1}{2},\frac{d-1}{2})}} \le \frac{1}{k^2}$. Now, for any $k$ that respects $-k \cdot \sqrt{\frac{\pi-3}{\pi d}} + \frac{2\cdot \sqrt d}{(d-1)\cdot \Beta(\frac{1}{2},\frac{d-1}{2})} > 0$ it holds that
$\Pr\brackets{\frac{d}{\norm{R_U \cdot \breve x}_1^2} \ge \frac{1}{  \parentheses{- k \cdot  \sqrt{\frac{\pi-3}{\pi d}} + \frac{2\cdot \sqrt d}{(d-1)\cdot \Beta(\frac{1}{2},\frac{d-1}{2})} }^2 }} \le \frac{1}{k^2}$. 
In addition, observe that $\Pr\brackets{\frac{d}{\norm{R_U \cdot \breve x}_1^2} >  d} = 0$. This means that,
\begin{align*}
\E \brackets{\frac{d}{\norm{R_U \cdot \breve x}_1^2}} \le \frac{d}{k^2} + \parentheses{1-\frac{1}{k^2}} \cdot \parentheses{- k \cdot  \sqrt{\frac{\pi-3}{\pi d}} + \frac{2\cdot \sqrt d}{(d-1)\cdot \Beta(\frac{1}{2},\frac{d-1}{2})} }^{-2}~.
\end{align*}
This concludes the proof.
\end{proof}

\begin{restatable}{lemma}{thmanbiaseddrivelargedeviation}\label{thm:anbiased_drive_large_deviation}
For any vector $x \in \mathbb R^d$, a constant $0<\epsilon<1$, and any dimension $d\ge2$, it holds that
\begin{align*}
 \E \brackets{\frac{d}{\norm{\mathcal R(\breve x)}_1^2}} \le 
\frac{\pi}{2}\cdot(1+\epsilon) + d \cdot exp \parentheses{-d \cdot \frac{\epsilon^2}{16(\pi-2)}}~.   
\end{align*}
\end{restatable}

\begin{proof}

By the definition of a uniform random rotation, it holds that $\frac{d}{\norm{\mathcal R(\breve x)}_1^2}$ follows the same distribution as $\frac{d \cdot \norm{Z}_2^2}{\norm{Z}_1^2}$ where $Z=(Z_1, \ldots, Z_d)$ and $Z_i \sim \mathcal N(0,1)$ are i.i.d. normal variables. Consider the two complementary events,
$A_- = {\norm{Z}_1 \le \beta \cdot d}$ and $A_+ = {\norm{Z}_1 > \beta \cdot d}$. Then,
\begin{equation*}
\begin{aligned}
\E \brackets{\frac{d \cdot \norm{Z}_2^2}{\norm{Z}_1^2}} = \E \brackets{\mathbbm{1}_{A_+} \cdot \frac{d \cdot \norm{Z}_2^2}{\norm{Z}_1^2}} + \E \brackets{\mathbbm{1}_{A_-} \cdot \frac{d \cdot \norm{Z}_2^2}{\norm{Z}_1^2}}.
\end{aligned}
\end{equation*}
We now treat each of the terms separately.
For the first term we have,
\begin{equation*}
\E \brackets{\mathbbm{1}_{A_+} \cdot \frac{d \cdot \norm{Z}_2^2}{\norm{Z}_1^2}} \le \frac{d}{(d \cdot \beta)^2} \cdot \E \brackets{\norm{Z}_2^2} = \frac{1}{\beta^2}. 
\end{equation*}
For the second term we have,
\begin{equation*}
\E \brackets{\mathbbm{1}_{A_-} \cdot \frac{d \cdot \norm{Z}_2^2}{\norm{Z}_1^2}} \le d \cdot \E \brackets{\mathbbm{1}_{A_-}} = d \cdot \Pr\brackets{\norm{Z}_1 \le \beta \cdot d}. 
\end{equation*}
Now, our goal is to bound the term $\Pr\brackets{\norm{Z}_1 \le \beta \cdot d}$. Denote $Y = (Y_1, \ldots, Y_d)$, where $Y_i = \sqrt{\frac{2}{\pi}} - \abs{Z_i}$. Then, it holds that,
\begin{equation*}
\Pr\brackets{\norm{Z}_1 \le \beta  d} = \Pr\brackets{\norm{Z}_1 - d  \sqrt{\frac{2}{\pi}} \le \beta  d- d  \sqrt{\frac{2}{\pi}}}  = \Pr\brackets{\sum_{i=1}^d Y_i \ge d  \parentheses{\sqrt{\frac{2}{\pi}} - \beta}}.
\end{equation*}
Further, for all $i$ we have that $\E \brackets{Y_i} = 0$, $\E \brackets{Y_i}^2 = 1-\frac{2}{\pi}$ and $Y_i \le \sqrt{\frac{2}{\pi}}$. According to Bernstein’s Inequality (from \cite{bennett1962probability}, or see \cite[Theorem 3.1]{maurer2003bound} for a simplified notation) this means that,
\begin{equation*}
\Pr\brackets{\sum_{i=1}^d Y_i \ge d  \parentheses{\sqrt{\frac{2}{\pi}} - \beta}} \le exp \parentheses{\frac{-d^2 \cdot (\sqrt{\frac{2}{\pi}} - \beta)^2}{2d \cdot (1-\frac{2}{\pi}) + d \cdot \frac{2}{3} \cdot  \sqrt{\frac{2}{\pi}} \cdot \parentheses{\sqrt{\frac{2}{\pi}} - \beta}}}. 
\end{equation*}
Setting $\beta = \sqrt{\frac{2 - \epsilon}{\pi}}$ for some $\epsilon>0$ yields,
\begin{equation*}
\Pr\brackets{\sum_{i=1}^d Y_i \ge d  \parentheses{\sqrt{\frac{2}{\pi}} - \beta}} \le exp \parentheses{-d \cdot \frac{3\cdot \parentheses{1-\sqrt{1-\frac{\epsilon}{2}}}^2}{3\pi - 4 - 2\sqrt{1-\frac{\epsilon}{2}}}}. 
\end{equation*}
Next, we obtain,
\begin{equation*}
\begin{aligned}
\E \brackets{\frac{d \cdot \norm{Z}_2^2}{\norm{Z}_1^2}} \le 
\frac{\pi}{2}\cdot\frac{1}{1-\frac{\epsilon}{2}} + d \cdot exp \parentheses{-d \cdot \frac{3\cdot \parentheses{1-\sqrt{1-\frac{\epsilon}{2}}}^2}{3\pi - 4 - 2\sqrt{1-\frac{\epsilon}{2}}}}.
\end{aligned}
\end{equation*}

We next use Taylor expansions to simplify the expression:
\begin{align*}
\E \brackets{\frac{d \cdot \norm{Z}_2^2}{\norm{Z}_1^2}} &\le 
\frac{\pi}{2}\cdot\frac{1}{1-\frac{\epsilon}{2}} + d \cdot exp \parentheses{-d \cdot \frac{3\cdot \parentheses{1-\sqrt{1-\frac{\epsilon}{2}}}^2}{3\pi - 4 - 2\sqrt{1-\frac{\epsilon}{2}}}}\\
&\le 
\frac{\pi}{2}\cdot\frac{1}{1-\frac{\epsilon}{2}} + d \cdot exp \parentheses{-d \cdot \frac{\parentheses{\frac{\epsilon}{2}}^2}{4(\pi-2)}} \\&\le 
\frac{\pi}{2}\cdot\parentheses{1+2\cdot\frac{\epsilon}{2}} + d \cdot exp \parentheses{-d \cdot \frac{\parentheses{\frac{\epsilon}{2}}^2}{4(\pi-2)}}
\\&=
\frac{\pi}{2}\cdot(1+\epsilon) + d \cdot exp \parentheses{-d \cdot \frac{\epsilon^2}{16(\pi-2)}}~.
\end{align*}
This concludes the proof.
\end{proof}

Finally, we prove the theorem, which we restate here.

\vNMSEofunbiaseddrive*

\begin{proof}
Recall that, using Lemma~\ref{cor:alg1_unbiased_vNMSE}, the vNMSE is $\E \brackets{\frac{d}{\norm{\mathcal R(\breve x)}_1^2}}-1$.
Therefore, the first part of the theorem follows from Lemma~\ref{thm:anbiased_drive_medium_deviation} with $k=\sqrt d$, which satisfies $\sqrt{k} < \frac{d}{d-1} \cdot \sqrt{\frac{\pi}{\pi - 3}} \cdot \frac{2}{\Beta(\frac{1}{2},\frac{d-1}{2})}$, as needed.
The second part of the theorem follows from setting $\epsilon = \sqrt{\frac{(24\pi-48)\ln d}{ d}}$ in Lemma~\ref{thm:anbiased_drive_large_deviation}. Notice that for $d\ge135$, we get $\epsilon<1$, as required.
Then:
\begin{align*}
\E \brackets{\frac{d}{\norm{\mathcal R(\breve x)}_1^2}} \le 
\frac{\pi}{2}\cdot(1+\epsilon) + d \cdot exp \parentheses{-1.5\ln d}
= 
\frac{\pi}{2} + \sqrt{\frac{{(6\pi^3-12\pi^2)}\cdot\ln d+1}{d}}
~.
\end{align*}
This concludes the proof.
\end{proof}

\subsection{Constant Scale for Unbiased Estimates in \alg with Uniform Random Rotations}\label{app:l1_expected_value}
In this appendix, we prove that it is possible to further lower the vNMSE while remaining unbiased by (deterministically) using $\sx=\frac{\norm{x}_2^2}{\mathbb E\brackets{\norm{\mathcal R_U(x)}_1}}=\frac{\norm{x}_2}{\mathbb E\brackets{\norm{T}_1}}$ where $T \in \mathcal S^{d-1}$ is uniformly at random distributed on the unit sphere \new{(observe that the unbiasedness proof in Theorem \ref{app:drive_is_unbiased} holds for this constant scale).} This is because $\frac{d}{\parentheses{\mathbb E\brackets{\norm{T}_1}}^2} - 1
\le \mathbb E\brackets{{\frac{d}{\norm{T}_1^2}}} {-} 1$.
This gives a provable $\frac{\pi}{2}-1$ vNMSE bound for any $d$. 
We start the analysis by proving two auxiliary lemmas. A generalization of these results is proved in~\cite[Eq. 3.3]{szablowski1998uniform}. \mbox{For completeness, we provide simplified proofs.}
\begin{lemma}\label{lem:rotated_coordinate_pdf}
Let $T=(T_1,\ldots,T_d)\in \mathcal S^{d-1}$ be a point on the unit sphere, drawn uniformly at random. Then, the PDF of $|T_i|$ is given by,
\begin{align*}
f_{|T_i|}(t) = \frac{2\cdot (1-t^2)^{\frac{d-1}{2}-1}}{\Beta(\frac{1}{2},\frac{d-1}{2})}.
\end{align*}
\end{lemma}

\begin{proof}

As was established in~\cite{muller1959note}, a uniformly random point on a sphere $T$ can be obtained by first deriving a random vector $Z = (Z_1,\ldots,Z_d)$ where $Z_i\sim N(0,1), i\in[1,\ldots,d]$ are normally distributed i.i.d. random variables, and then normalizing its norm by $T = \frac{Z}{\norm{Z}_2}$.

Therefore, the i'th coordinate of $T$ is given as $T_i=\frac{Z_i}{\norm{Z}_2}$.
Since the coordinate distribution is sign-symmetric, we focus on its distribution in the interval $[0, 1]$. Observe that for $T_i\in[0,1]:$
\begin{equation*}
\begin{aligned}
\Pr\brackets{|T_i|\le t} &= \Pr\brackets{T_i^2 \le t^2} =  \Pr\brackets{\frac{Z_i^2}{\norm{Z}_2^2}\le t^2} =
\Pr\brackets{\frac{Z_i^2}{\norm{Z}_2^2-Z_i^2}\le t^2\cdot \frac{\norm{Z}_2^2}{\norm{Z}_2^2-Z_i^2}} \\
&=
\Pr\brackets{\frac{Z_i^2}{\norm{Z}_2^2-Z_i^2}\le t^2\cdot\parentheses{1+\frac{Z_i^2}{\norm{Z}_2^2-Z_i^2}}} =
\Pr\brackets{\frac{Z_i^2}{\norm{Z}_2^2-Z_i^2}\le \frac{t^2}{1-t^2}}. 
\end{aligned}    
\end{equation*}
Therefore
$\Pr\brackets{|T_i|\le t} = \Pr\brackets{\frac{(d-1)Z_i^2}{\norm{Z}_2^2-Z_i^2}\le \frac{(d-1)t^2}{1-t^2}}$. Notice that $Z_i^2\sim\chi^2(1)$ and $(\norm{Z}_2^2-Z_i^2)\sim\chi^2(d-1)$ are independent $\chi^2$ random variables. 
Now denote $Y\triangleq\frac{(d-1)Z_i^2}{\norm{Z}_2-Z_i^2}$. Observe that $Y$ follows a $F_{1,d-1}$ distribution. Therefore, the CDF of $Y$ is then given by the following expression:
$\Pr\brackets{Y \le y}=I_{\frac{y}{y+d-1}}(\frac{1}{2},\frac{d-1}{2})$, where $I$ is the regularized incomplete Beta function. Substituting $y=\frac{(d-1)t^2}{1-t^2}$ yields 
$\Pr\brackets{|T_i|\le t}= I_{t^2}(\frac{1}{2},\frac{d-1}{2})$. 
In particular, taking the derivative yields 
$f_{|T_i|}(t) = \frac{2\cdot (1-t^2)^{\frac{d-1}{2}-1}}{\Beta(\frac{1}{2},\frac{d-1}{2})} $. 
Interestingly, this PDF is similar to the Beta distribution (not to be confused with the $\Beta$ function). Indeed, one can verify that $T_i'=\frac{T_i+1}{2}$ follows a Beta distribution with $T_i'\sim \mathit{Beta}(\frac{d-1}{2},\frac{d-1}{2})$.
%
%
%
%
%
%
\end{proof}
With Lemma \ref{lem:rotated_coordinate_pdf} at hand, we obtain the following result.

\begin{restatable}{lemma}{loneexpectedvalue}\label{lem:l1_expected_value}
Let $T \in \mathcal S^{d-1}$ be a point on the unit sphere, drawn uniformly at random. Then, 
\begin{align*}
 \E\brackets{\norm{T}_1} = \frac{2d}{(d-1)\cdot \Beta(\frac{1}{2},\frac{d-1}{2})}~.   
\end{align*}
\end{restatable}

\begin{proof}
Due to the linearity of expectation we have that $\E\brackets{\norm{T}_1} = \sum_{i=1}^d \E\brackets{\abs{T_i}}$. Thus, it is sufficient to calculate the expected value of a single element in the sum. By Lemma \ref{lem:rotated_coordinate_pdf} we have that $\E\brackets{\abs{T_i}} = \int_0^1 t \cdot \frac{2\cdot (1-t^2)^{\frac{d-1}{2}-1}}{\Beta(\frac{1}{2},\frac{d-1}{2})} dt = \frac{2}{(d-1)\cdot \Beta(\frac{1}{2},\frac{d-1}{2})}$. \mbox{Summing over $i \in \set{1,\ldots,d}]$ yields the result.\qedhere}
\end{proof}

We note that $\E\brackets{\norm{T}_1}=\frac{2d}{(d-1)\cdot \Beta(\frac{1}{2},\frac{d-1}{2})} \ge \sqrt{\frac{2d}{\pi}}$.
Finally, recall that our vNMSE is bounded by ${\frac{d}{\mathbb E\brackets{\norm{T}_1^2}}} {-} 1$, which is therefore at most $\frac{\pi}{2}-1$. 
In practice, while this deterministic approach gives a lower vNMSE for low dimensions, the benefit is marginal even for $d$ values of several hundreds.


\subsection{Proof of Theorem \ref{cor:dme}}\label{app:dme}

\cordme*

\begin{proof}
We have that, 
\begin{equation*}
\begin{aligned}
\E\brackets{\norm{x_{\text{avg}}-\widehat {x_{\text{avg}}}}_2^2} &= \frac{1}{n^2} \cdot  \E\brackets{\norm{\sum_{\mathfrak c=1}^{n}x_{(\mathfrak c)}-\sum_{\mathfrak c=1}^{n}\widehat{x_{(\mathfrak c)}}}_2^2} = \frac{1}{n^2} \cdot \sum_{\mathfrak c,\mathfrak c'} \E \brackets{\angles{x_{(\mathfrak c)} - \widehat{x_{(\mathfrak c)}}, x_{(\mathfrak  c')} - \widehat{x_{(\mathfrak  c')}}} } \\  &= \frac{1}{n^2} \cdot \sum_{\mathfrak c} \E \brackets{\angles{x_{(\mathfrak c)} - \widehat{x_{(\mathfrak c)}}, x_{(\mathfrak c)} - \widehat{x_{(\mathfrak c)}}} } 
\\&\qquad+ \frac{1}{n^2} \cdot \sum_{\mathfrak c \neq \mathfrak c'} \E \brackets{\angles{x_{(\mathfrak c)} - \widehat{x_{(\mathfrak c)}}, x_{(\mathfrak  c')} - \widehat{x_{(\mathfrak  c')}}} } \\&\underset{\mbox{\tiny{$\angles{1}$}}}{{=}}~\frac{1}{n^2} \sum_{\mathfrak c} \E \brackets{\norm{x_{(\mathfrak c)}-\widehat{x_{(\mathfrak c)}}}_2^2} = 
\frac{1}{n^2} \cdot \sum_{\mathfrak c} \norm{x_{(\mathfrak c)}}_2^2 \cdot \E \brackets{\frac{\norm{x_{(\mathfrak c)}-\widehat{x_{(\mathfrak c)}}}_2^2}{\norm{x_{(\mathfrak c)}}_2^2}} 
\\&=
\frac{1}{n^2} \cdot \sum_{\mathfrak c} \norm{x_{(\mathfrak c)}}_2^2 \cdot vNMSE = \frac{vNMSE}{n} \cdot \frac{\sum_{\mathfrak c} \norm{x_{(\mathfrak c)}}_2^2}{n}. 
\end{aligned}    
\end{equation*}
Here, since the $\E \brackets{\widehat{x_{(\mathfrak c)}}} =x_{(\mathfrak c)}$ for all $i$ and since $\widehat{x_{(\mathfrak c)}}$ and $\widehat{x_{(\mathfrak  c')}}$ are independent for all $\mathfrak c \neq \mathfrak c'$, $\angles{1}$ follows from 
$\E \brackets{\angles{x_{(\mathfrak c)} - \widehat{x_{(\mathfrak c)}}, x_{(\mathfrak  c')} - \widehat{x_{(\mathfrak  c')}}} } = \E \brackets{\angles{x_{(\mathfrak c)}, x_{(\mathfrak  c')}}} - \E \brackets{\angles{x_{(\mathfrak c)} ,\widehat{x_{(\mathfrak  c')}}} } - \E \brackets{\angles{\widehat{x_{(\mathfrak c)}}, x_{(\mathfrak  c')}}} + \E \brackets{\angles{\widehat{x_{(\mathfrak c)}}, \widehat{x_{(\mathfrak  c')}}}} = 0$. 
\end{proof}

With this result at hand, we can derive useful theoretical guarantees, especially for high-dimensional vectors (e.g., gradient updates in NN federated and distributed training). For example

\begin{corollary}
For $n$ clients with arbitrary vectors in $\mathbb R^d$ and $d {\,\ge\,} 10^5$, the NMSE is smaller than $\frac{0.673}{n}$.
\end{corollary}

\section{Message Representation Length}\label{app:MessageRepresentationLength}
Here we discuss the length of the message $\big(\sx,\text{sign}(\mathcal R(x))\big)$ that \sender sends to \receiver. The message includes the sign vector $\text{sign}(\mathcal R(x))$, which takes exactly $d$ bits.
It remains to consider how $\sx$ is to be represented and how that affects the vNMSE of our algorithms.
We first note that one must assume that the norm of the sent vector $x$ is bounded, as otherwise no finite number of bits can be used for the transmission.
This is justified as the coordinates of the vector are commonly represented by a floating point representation with constant length (e.g., using $32$ bits). In particular, this assumption implies that $\norm{x}_2=O(d)$, with each coordinate is bounded by a constant. 
It follows that all the quantities we consider for the scale satisfy $\sx=O(d)$. 
This is because the scale of Theorem~\ref{thm:biased_drive_vNMSE} satisfies $\frac{\norm{\mathcal R(x)}_1}{d}\le \frac{\sqrt d\cdot \norm{\mathcal R(x)}_2}{d} = \frac{\norm{x}_2}{\sqrt d} = O(\sqrt d)$ 
(where here the first step uses the Cauchy-Schwarz inequality), and similarly the scale of Theorem~\ref{theorem:drive_is_unbised} satisfies
$\frac{\norm{x}_2^2}{\norm{\mathcal R(x)}_1}\le \frac{\norm{x}_2^2}{\norm{\mathcal R(x)}_2} =\norm{x}_2 = O(d)$.

As shown in~\cite{pmlr-v70-suresh17a}, this means that one can use $O(\log d)$ bits to represent $\sx$ to within a $d^{-\Omega(1)}$ additive error, i.e., a factor that rapidly drops to zero as $d$ grows.
This increases the vNMSE for a single vector under \alg by at most $d^{-\Omega(1)}$, i.e., an additive factor that diminishes as $d$ grows. 

In the distributed mean estimation setting, where $n$ clients send their vectors for averaging, \cite{pmlr-v70-suresh17a} suggested using $O(\log(nd))$ bits.
In fact, the dependency on $n$ may be dropped. Specifically, $O(\log d)$ bits suffice if we encode $\sx$ \emph{in an unbiased manner}.
One simple way to do so is to encode $\floor S$ exactly (which requires $O(\log d)$ bits since $\sx = O(d)$) and to encode the remainder $r = \sx-\floor{\sx}$ using $O(\log d)$-bit stochastic quantization. (That is, express $r$ using the $O(\log d)$ bits but use randomized rounding for the last bit.)  This approach still yields a vNMSE increase of $d^{-\Omega(1)}$ for each vector.
As for the distributed mean estimation task, it follows that Theorem~\ref{cor:dme} still holds for \alg as the vectors are sent in an unbiased \mbox{manner. Thus the NMSE increases by only $(nd)^{-\Omega(1)}$.}


The above argument means that \alg can use $d+O(\log d) = d(1+o(1))$ bit messages. While in practice implementations use 32- or 64-bit representations for $\sx$ thus send $d+O(1)$ bit messages, we use the $d(1+o(1))$ notation to ensure compatibility with the theoretical guarantees.

%
%
%
Similar arguments show that our improved \algp algorithm, presented in Section~\ref{sec:drive_plus}, can also be encoded using $d+O(\log d)=d(1+o(1))$ bits.

\section{Supplementary Material for \algp}\label{app:drive_plus_app}

In this section, we provide additional information about the \algp algorithm. 

\subsection{Pseudocode}
Here, we provide the pseudocode for \algp in Algorithm~\ref{code:alg2}:
\begin{algorithm}[H]
\caption{~\algp}
\label{code:alg2}
\begin{multicols}{2}
\begin{algorithmic}[1]
  \Statex \hspace*{-4mm}\textbf{\sender:}
  \vspace*{1mm}
  \State Compute $\mathcal R(x)$.
  \State Compute the $c_0,c_1\in\mathbb R$ values that minimize 
  \begin{align*}
   \sum_{i=1}^d {\min\set{(\mathcal R(x)_i-c_0)^2, (\mathcal R(x)_i-c_1)^2}}.   
  \end{align*}
  \State Compute $X\in\set{0,1}^d$ such that
\begin{align*}
\hspace*{-4mm}
X_i = \begin{cases}
0 & \mbox{if $|\mathcal R(x)_i-c_0|\le |\mathcal R(x)_i-c_1|$}\\
1 & \mbox{otherwise}
\end{cases}.    
\end{align*}
\State Compute $\sc$ and  $\overline c_0 = \sc\cdot c_0,\ \overline c_1 = \sc\cdot c_1$.
  \State Send $\big(\overline c_0,\overline c_1,X\big)$ to \receiver.
\end{algorithmic}
\columnbreak
\begin{algorithmic}[1]
\Statex \hspace*{-4mm}\textbf{\receiver:}
\State Compute $\widehat {\mathcal R(x)}\in\set{\overline c_0,\overline c_1}^d$ such that 
\begin{align*}
 \widehat {\mathcal R(x)}_i = \overline c_{X_i}.\label{line:alg2Rhatx}   
\end{align*}
\State Estimate $\widehat x = \mathcal R^{-1}(\widehat {\mathcal R(x)})$.
\end{algorithmic}
\end{multicols}
\end{algorithm}

We now provide the vNMSE guarantees for \algp. 

\subsection{\textbf{1b~-~Vector Estimation with \algp}}\label{app:1bDrive+}

We provide the following lemma from which it follows that the vNMSE of \algp with $\sc=1$ is upper-bounded by that of \alg when both algorithms are not required to be unbiased.

\begin{lemma}\label{lemma:alg2isbetter}
For any $R \sim \mathcal R$ and any $x\in\mathbb R^d$, the SSE of \algp with $\sc=1$ is upper-bounded by the SSE of \alg for any choice of $\sx$.
\end{lemma}
\begin{proof}
By definition, the values $c_0,c_1$ respect that $\sum_{i=1}^d {\min\set{(\mathcal R(x)_i-c_0)^2, (\mathcal R(x)_i-c_1)^2}} \le \sum_{i=1}^d {\min\set{(\mathcal R(x)_i-\sx)^2, (\mathcal R(x)_i+\sx)^2}}$ for any choice of $\sx$. Now, recall that the SSE in estimating $\mathcal R(x)$ using $\widehat {\mathcal R(x)}$ equals that of estimating $x$ using $\widehat x$. This concludes the proof. 
\end{proof}

\subsection{\textbf{1b~-~Distributed Mean Estimation with \algp}}
In this subsection, we prove that \algp can be made unbiased, when using the uniform random rotation $\mathcal{R}_U$, by the right choice of scale $\sc$, and that its vNMSE is upper bounded by that of \alg.

Let $c \in \set{c_0,c_1}^d$ where $c_0,c_1\in\mathbb R$ minimize the term $\sum_{i=1}^d (\mathcal R(x)_i-c_i)^2$. Namely, $c_0,c_1$ are the solution to the 2-means optimization over the entries of the rotated vector.

\begin{restatable}{theorem}{driveplusunbiased}\label{thm:drive_plus_unbiased}
For any $x \in \mathbb R^d$ set $\sc = \frac{\norm{x}_2^2}{\norm{c}_2^2}$. Then, $\E [\widehat x] = x$. Thus, Theorem \ref{cor:dme} applies to \algp.
\end{restatable}

\begin{proof}

For ease of exposition, for any vector $y \in \mathbb R^d$, we denote by $c(y)$ the vector that minimizes the term $\sum_{i=1}^d (y_i-c_i)^2$ where $c(y) \in \set{c_0,c_1}^d$ and $c_0,c_1\in\mathbb R$.

For the vector $\mathcal R (x)$ we drop the notation and simply use $c$. Namely, $c \in \set{c_0,c_1}^d$ where $c_0,c_1\in\mathbb R$ is a vector that minimizes the term $\sum_{i=1}^d (\mathcal R(x)_i-c_i)^2$.

The proof follows similar lines to that of Theorem~\ref{theorem:drive_is_unbised}.

For any $x \in \mathbb R^d$ denote $x'=(\norm x_2,0,\ldots,0)^T$
and let $R_{x\shortrightarrow x'}\in\mathbb R^{d\times d}$ be a rotation matrix such that $R_{x\shortrightarrow x'}\cdot x = x'$.
Further, denote $R_{x} = R_U R_{x \shortrightarrow x'}^{-1}$. 
Using these definitions we have that,
\begin{equation*}
\begin{aligned}
    \widehat x \myeq{1} & R_{x\shortrightarrow x'}^{-1} \cdot R_{x\shortrightarrow x'}\cdot\widehat x
    \myeq{2}
    R_{x\shortrightarrow x'}^{-1} \cdot R_{x\shortrightarrow x'} \cdot  R_U^{-1}\cdot\sc\cdot c\parentheses{R_U\cdot x}   \\
    \myeq{3}&R_{x\shortrightarrow x'}^{-1} \cdot R_x^{-1}\cdot{\sc\cdot c\parentheses{R_x\cdot R_{x\shortrightarrow x'}\cdot x}} 
    \myeq{4} \sc\cdot R_{x\shortrightarrow x'}^{-1} \cdot R_x^{-1}\cdot{c\parentheses{R_x\cdot x'}} \\
    \myeq{6}& R_{x\shortrightarrow x'}^{-1}\cdot \frac{\norm{x}_2^2 \cdot R_x^{-1}\cdot{ c\parentheses{R_x\cdot x'}}}{\norm{c}_2^2}~. 
\end{aligned}    
\end{equation*}

Again, let $\matcol{i}$ be a vector containing the values of the $i$'th column of $R_x$. Then, $R_x \cdot x'=\norm x_2 \cdot \matcol{0}$ and therefore $c(R_x \cdot x') = c(\norm x_2 \cdot \matcol{0}) = \norm x_2 \cdot c(\matcol{0})$. 
We obtain,
\begin{equation*}
\begin{aligned}
    R_x^{-1}\cdot{ c\parentheses{R_x\cdot x'}} = \norm x_2 \cdot \parentheses{\norm{c(\matcol{0})}_2^2,\angles{\matcol{1},c(\matcol{0})},\ldots,\angles{\matcol{d-1},c(\matcol{0})}}^T, 
\end{aligned}   
\end{equation*} 
Next, we have that,
\begin{equation}\label{eq:drive_plus_before_expectation}
\begin{aligned}
\widehat x = R_{x\shortrightarrow x'}^{-1} \cdot \norm{x}_2 \cdot \parentheses{\frac{\norm{x}_2^2\cdot\norm{c(\matcol{0})}_2^2}{\norm{c}_2^2},\frac{\norm{x}_2^2\cdot\angles{\matcol{1},c(\matcol{0})}}{\norm{c}_2^2},\ldots,\frac{\norm{x}_2^2\cdot\angles{\matcol{d-1},c(\matcol{0})}}{\norm{c}_2^2}}^T~.
\end{aligned}   
\end{equation} 
Observe that $\norm{x}_2^2\cdot\norm{c(\matcol{0})}_2^2 = \norm{c(R_x\cdot x')}_2^2 = \norm{c(R_U\cdot x)}_2^2 = \norm{c}_2^2$. This means that the first coordinate in the above vector is $1$. 


We continue the proof by using the same construction as in the proof of Theorem~\ref{theorem:drive_is_unbised}.

In particular, consider an algorithm \alg$'$ that operates exactly as \alg but, instead of directly using the sampled rotation matrix $R_U = R_x \cdot R_{x \shortrightarrow x'}^{-1}$ it calculates and uses the rotation matrix $R_U' = R_x \cdot I' \cdot R_{x \shortrightarrow x'}^{-1}$ where $I'$ is identical to the $d$-dimensional identity matrix with the exception that $I'_{00}=-1$ instead of $1$. 

Now, since $R_U \sim \mathcal R_U$ and $R_{x\shortrightarrow x'}$ is a fixed rotation matrix, we have that $R_{x} \sim \mathcal R_U$. In turn, this also means that $R_{x} \cdot I' \sim \mathcal R_U$ since $I'$ is a fixed rotation matrix.

Consider a run of both algorithm where $\widehat x$ is the reconstruction of \alg for $x$ with a sampled rotation $R_U$ and $\widehat x'$ is the corresponding reconstruction of \alg$'$ for $x$ with the rotation $R_U'$. 

According to \eqref{eq:drive_plus_before_expectation} it holds that:
$\widehat x + \widehat x' =R_{x\shortrightarrow x'}^{-1} \cdot \norm{x}_2 \cdot \parentheses{2,0,\ldots,0}^T = 2 \cdot x$. This is because both runs are identical except that the first column of $R_{x}$ and $R_{x} \cdot I'$ have opposite signs and thus the resulting centroids $c(\matcol{0})$ also flip signs in the run of \alg$'$. In particular, it holds that $\E \brackets{\widehat x + \widehat x'} = 2 \cdot x$. But, since $R_{x} \sim \mathcal R_U$ and $R_{x} \cdot I' \sim \mathcal R_U$, both algorithms have the same expected value. This yields  $\E \brackets{\widehat x} = \E \brackets{\widehat x'} = x$.
%
%
This concludes the proof.
\end{proof}

To prove the SSE of \algp is lower-bounded by that of \alg, we use the following observation.

\begin{observation}\label{obs:centroids}
For $i\in\set{0,1}$, let $I_i=\set{j\in\set{1,\ldots,d}\mid |\mathcal R(x)_j-c_i|\le|\mathcal R(x)_j-c_{1-i}|}$ denote the set of points closer to $c_i$. Then $c_i = \frac{\sum_{j\in I_i} \mathcal R(x)_j}{|I_i|}$.
\end{observation}

We are now ready to prove the bound on the vNMSE \algp.

\begin{restatable}{lemma}{driveplusunbiasedvNMSE}\label{lem:drive_plus_unbiased_vNMSE}
For any $d$ and $R_U$, the SSE of \algp with $\sc = \frac{\norm{x}_2^2}{\norm{c}_2^2}$ is upper-bounded by the SSE of \alg with $\sx=\frac{\norm{x}_2^2}{\norm{\mathcal R(x)}_1}$. Thus, Theorem \ref{thm:vNMSEofunbiaseddrive} and Corollary \ref{cor:deviation_res_3} apply to \algp.
\end{restatable}
\begin{proof}
Recall that $c \in \set{c_0,c_1}^d$, where $c_0,c_1\in\mathbb R$, minimizes the term $\sum_{i=1}^d (\mathcal R(x)_i-c_i)^2$ and that the quantization SSE of the rotated vector equals the SSE of estimating $x$ using $\widehat x$. 
For \algp, this means that the SSE is given by 
\begin{equation*}
\begin{aligned}
\norm{\mathcal R(x)- \sc \cdot c}_2^2 &= \norm{\mathcal R(x)}_2^2 - 2\cdot\sc\cdot\angles{\mathcal R(x), c} + \norm{\sc\cdot c}_2^2 \\&= \norm{x}_2^2 - 2\cdot\sc\cdot \angles{\mathcal R(x), c} + (\sc)^2\cdot \norm{c}_2^2
\\&= \norm{x}_2^2 - 2\cdot\sc\cdot\norm{c}_2^2 + (\sc)^2\cdot \norm{ c}_2^2.
\end{aligned}
\end{equation*}

According to Observation~\ref{obs:centroids}, the inner produce in the first coordinate satisfies:
\begin{equation*}
 \angles{\mathcal R(x),c} = \sum_{j\in I_0} \mathcal R(x)_j\cdot c_{0} + \sum_{j\in I_1} \mathcal R(x)_j\cdot c_{1}   = {c_0^2}\cdot{|I_0|} + {c_1^2}\cdot{|I_1|} = \norm{c}_2^2.    
\end{equation*}
We we will show that for any vector $x\in\mathbb R^d$ and any $R_U$, the SSE of \algp with $\sc = \frac{\norm{x}_2^2}{\norm{c}_2^2}$ is upper bounded by the SSE of \alg with $\sx = \frac{\norm{x}_2^2}{\norm{\mathcal R_U(x)}_1}$. Both results immediately follow.

From the above and Theorem \ref{thm:theoreticalAlg}, it is sufficient to show that $\norm{c}_2^2 \brackets{(\sc)^2-2\sc} \le  - 2 \cdot \sx \cdot \norm{\mathcal R(x)}_1 + \sx^2 \cdot d$, which, by substituting $\sx$ and $\sc$, is equivalent to
$\frac{\norm{x}_2^4}{\norm{c}_2^2} \le \frac{\norm{x}_2^4}{\norm{\mathcal R(x)}_1^2}$ or $\norm{c}_2 \le \norm{\mathcal R(x)}_1$. The last inequality immediately holds since $\norm{c}_2 \le \norm{c}_1$ and $ \norm{c}_1 \le \norm{\mathcal R(x)}_1$ by the triangle inequality.
\end{proof}

Similarly to \alg, one can deterministically set $\sc=\frac{\norm{x}_2^2}{\mathbb E\brackets{\norm{c}_2^2}}$ and gain a $\frac{\pi}{2}-1$ upper bound on the vNMSE for any dimension $d$. 
Notice that this requires calculating this expression;
for any dimension $d$, one can approximate this quantity once to within arbitrary precision using a Monte-Carlo simulation. Nevertheless, as mentioned, this is less favorable as it introduces undesirable computational overhead for non uniform rotations. 

\subsection{Implementation Notes}
\new{We note that it is possible to introduce a tradeoff between the additional complexity introduced by the computation of the centroids and the reduction in vMNSE compared to \alg. This is achieved by initializing the centroids symmetrically to that of \alg and continuing to apply iterations of Lloyd's algorithm. Each such additional iteration introduces more calculations but reduces the SSE.} 

\new{In high dimensions, we find that the SSE of \alg and \algp similar for vectors whose distribution admits finite moments (they converge to the same value). In low dimensions, we often find that after 2-3 such iterations, further improvement is negligible. Since Lloyd's algorithm admits an efficient GPU implementation, these extra iterations may offer a better vMNSE to computation tradeoff than optimally finding the centroids using less GPU-friendly implementations.}




\section{Supplemental Results for Structured Random Rotations}

We next provide supplementary results for \alg and \algp with a structured random rotation.

\subsection{Convergence of the Rotated Coordinates to a Normal Distribution}\label{app:proofOfBerryEsseen}
For completeness, we restate Lemma~\ref{lem:proofOfBerryEsseen}:
\proofOfBerryEsseen*

\begin{proof}
By definition $\frac{1}{\sqrt{d}} \cdot \mathcal R_H(x)_i = \frac{1}{d} \cdot \sum_{j=1}^d x_j H_{ij} D_{jj}$ for all $i$. Due to the sign symmetry of $D_{jj}$, it holds that $x_j H_{ij} D_{jj}$ are zero-mean i.i.d. random variables. Also, $\E \brackets{x_j H_{ij} D_{jj}}^2 = \sigma^2$ and $\E \brackets{\abs{x_j H_{ij} D_{jj}}}^3 = \rho$. 
Now, denote by $F_{i,d}$ be the cumulative distribution function (CDF) of $\frac{1}{\sigma} \cdot \mathcal R_H (x)_i$.
Using the Berry–Esseen theorem~\cite{esseen1956moment}, we obtain than
$\sup_{x\in \mathbb R}\abs{F_{i,d}(x)-\Phi(x)} \le \frac{0.409\cdot \rho}{\sigma^3 \sqrt d}$, where $\Phi$ the CDF of the standard normal distribution. This concludes the proof.
\end{proof}

\subsection{Analysis of the 4th Moment When Using a Structured Random Rotation}\label{app:Hadamard_4th_moments}

Further assume that $\E[x_j^4] = M_4 < \infty$. Then, we have that, 
\begin{equation*}
\begin{aligned}
\E &\brackets{(\mathcal R_H \cdot x)_{i_1} \cdot (\mathcal R_H \cdot x)_{i_2} \cdot (\mathcal R_H \cdot x)_{i_3} \cdot (\mathcal R_H \cdot x)_{i_4}}\\ &
= \frac{1}{d^2} \sum_{j=1}^d \sum_{k=1}^d \sum_{l=1}^d \sum_{m=1}^d \E \brackets{x_j H_{i_1j} D_{jj} \cdot  x_k H_{i_{2}k} D_{kk} \cdot  x_l H_{i_{3}l} D_{ll} \cdot  x_m H_{i_{4}m} D_{mm}}\\ &
= \frac{1}{d^2} \sum_{j=1}^d \sum_{k=1}^d \sum_{l=1}^d \sum_{m=1}^d   H_{i_{1}j} H_{i_{2}k} H_{i_{3}l} H_{i_{4}m} \cdot \E \brackets{x_j x_k  x_l x_m \cdot D_{jj} D_{kk} D_{ll} D_{mm}}.
\end{aligned}
\end{equation*}
There are only two cases for which the expectation is not zero. Where there are two different indexes (where the resulting expectation is $\sigma^4$) and where all four are identical (where the resulting expectation is $M_4$). Therefore,
\begin{equation*}
\begin{aligned}
\E &\brackets{(\mathcal R_H \cdot x)_{i_1} \cdot (\mathcal R_H \cdot x)_{i_2} \cdot (\mathcal R_H \cdot x)_{i_3} \cdot (\mathcal R_H \cdot x)_{i_4}}\\ &
= \frac{1}{d^2} \cdot \sigma^4 \cdot\brackets{\sum_{j=1}^d H_{1+(i_1-1) \oplus (i_2-1), j} \cdot \brackets{\sum_{k=1, k \neq j}^d H_{1+(i_3-1) \oplus (i_4-1), k}}} 
\\&\qquad+ \frac{1}{d^2} \cdot\sigma^4\cdot \brackets{\sum_{j=1}^d H_{1+(i_1-1) \oplus (i_3-1), j} \cdot \brackets{\sum_{k=1, k \neq j}^d H_{1+(i_2-1) \oplus (i_4-1), k}}} \\&
+ \frac{1}{d^2} \cdot\sigma^4 \cdot\brackets{\sum_{j=1}^d H_{1+(i_1-1) \oplus (i_4-1), j} \cdot \brackets{\sum_{k=1, k \neq j}^d H_{1+(i_2-1) \oplus (i_3-1), k}}}
\\&\qquad
+ \frac{1}{d^2}\cdot M_4 \cdot\sum_{j=1}^d H_{1+(i_1-1) \oplus (i_2-1) \oplus (i_3-1) \oplus (i_4-1), j}
\end{aligned}
\end{equation*}
Now we can calculate all the fourth moments. 

There are four cases to consider: (1) for $i_1=i_2=i_3=i_4$ we obtain $3\cdot \sigma^4 \cdot \frac{d\cdot(d-1)}{d^2}  + \frac{d}{d^2} \cdot M_4 = 3\cdot \sigma^4 + O(\frac{1}{d})$; (2) for all pair equalities (e.g., $i_1=i_2$ and $i_3=i_4$ but $i_1 \neq i_2$) we obtain $\sigma^4 \cdot \frac{d\cdot(d-1)}{d^2} + \frac{d}{d^2} \cdot M_4 = \sigma^4 + O(\frac{1}{d})$; (3) when having a unique entry and at least a pair that equals (e.g., $i_1=i_2,i_2 \neq i_4$ and $i_3\neq i_4$ regardless of whether $i_2=i_3$) we immediately obtain zero (recall that all odd moments are zero); (4) for all distinct entries it may be the case that $1+ (i_1-1) \oplus (i_2-1) \oplus (i_3-1) \oplus (i_4-1) = 1$, therefore we obtain $\frac{d}{d^2} \cdot M_4 = O(\frac{1}{d})$.

These indeed approach, with a rate of $\frac{1}{d}$, the 4'th moments of \emph{independent} normal random variables.

\section{Additional Simulation Details and Results}\label{appendix:additional_simulations}

In this section we provide additional details and simulation results. 
\ifdefined\arxiv
\else
The source code and full reproduction instructions are in the supplementary material archive.
\fi

\subsection{Preexisting Code and Datasets Description}\label{app:utilizedAssets}

\T{Code.} Our federated learning experiments use the TensorFlow Federated \cite{tensorflowfed} library with the Hadamard rotation and the Kashin's representation code taken from the TensorFlow Model Optimization Toolkit \cite{tensorflowmo}. Tensorflow and its subpackages are licensed under the Apache License, Version 2.0 \cite{apache2}. The rest of the experiments use PyTorch \cite{NIPS2019_9015}, which is licensed under a BSD-style license \cite{pytorchlicense}.

\T{Datasets.} Next we provide a quick overview of the datasets used:

\textit{MNIST}~\cite{lecun1998gradient, lecun2010mnist}. Includes 70,000 examples of $28{\times}28$ grayscale labeled images of handwritten \mbox{digits (10 classes)}.

\textit{EMNIST}~\cite{cohen2017emnist}. Extends MNIST and includes 749,068 examples of $28{\times}28$ grayscale labeled images of 62 handwritten characters.

\textit{CIFAR-10} and \textit{CIFAR-100} \cite{krizhevsky2009learning}. Both consist of 60,000 examples of $32{\times}32$ images with 3 color channels, in each dataset the images are labeled into 10 and 100 classes, respectively.

\textit{Shakespeare} \cite{shakespeare}. Contains 18,424 examples, where each example is a contiguous set of lines spoken by a character in a play by Shakespeare.
 
\textit{Stack Overflow} \cite{stackoverflowdb}. Consists of 152,404,765 examples, each containing the text of either an answer or a question.
The data consists of the body text of all questions and answers

For the federated learning experiments, following previous works in the field \cite{mcmahan2017communication, caldas2019leaf}, the EMNIST, Shakespeare, and Stack Overflow datasets are partitioned naturally to create an heterogeneous unbalanced client dataset distribution. Specifically, they are partitioned by character writer, play speaker, and Stack Overflow user, respectively. In the CIFAR-100 federated experiment, the examples are split among 500 clients using a random heterogeneity-inducing process described in \cite[Appendix~C.1]{reddi2020adaptive}.

All datasets were released under  CC BY-SA 3.0 \cite{CCASA3} or similar licenses, the exact details are in the cited references.
 
 


\begin{figure}[h]
\centering
\centerline{\includegraphics[width=\textwidth, trim=0 40 0 0, clip]{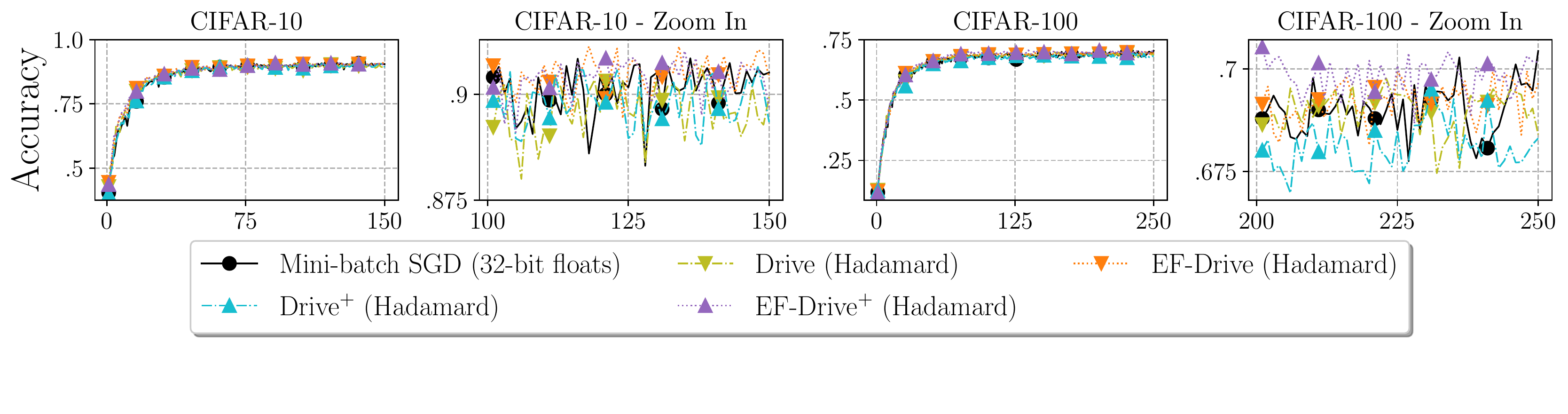}}

\caption{Distributed CNN EF experiments comparing \alg and \algp with and without EF. Accuracy per round on distributed learning tasks, with a zoom-in on the last 50 rounds.}

\label{fig:ef:drive}
\end{figure}

\begin{figure}[h]
\centering
\centerline{\includegraphics[width=\textwidth, trim=0 40 0 0, clip]{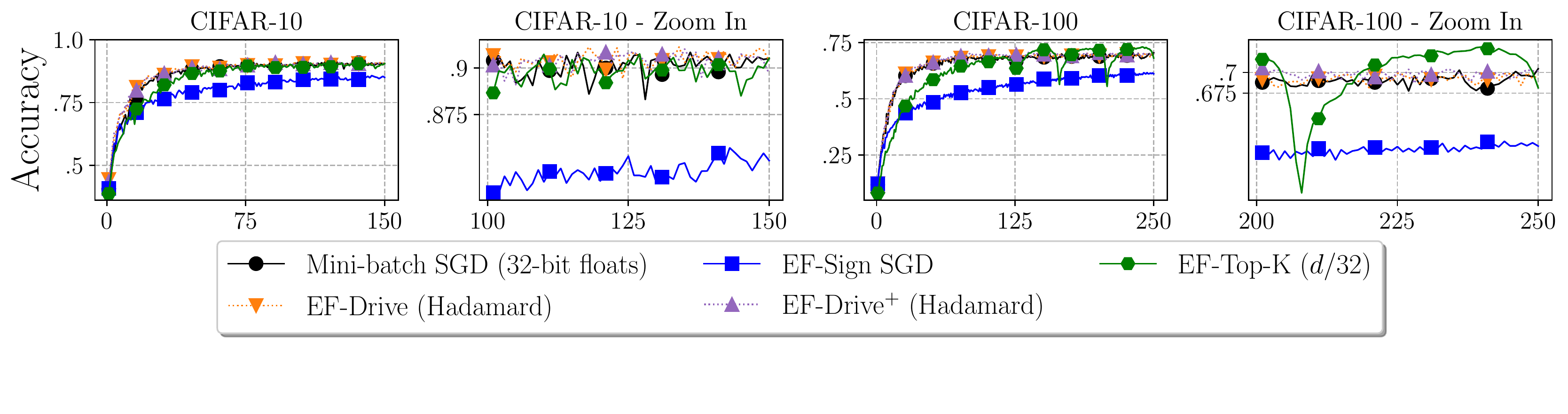}}

\caption{Distributed CNN EF experiments comparing \alg and \algp with and techniques utilizing EF. Accuracy per round on distributed learning tasks, with a zoom-in on the last 50 rounds.}

\label{fig:ef:others}
\end{figure}

\subsection{\new{Error-Feedback Experiments}}\label{app:exp:ef}

We next conduct error feedback (EF) experiments for the distributed CNN training. For \alg (and similarly for \algp), we use the following formula to set the scale, $\sx=\min{\{  2 \cdot \frac{\norm{\mathcal R_H (x)}_1}{d},  \frac{\norm{x}_2^2}{\norm{\mathcal R_H(x)}_1}  \}}$. This scale ensures that the \emph{compressor assumption}~\cite[Assumption A]{karimireddy2019error}  holds but tries to make use of the ``unbiased'' scale whenever possible. In fact, in our CNN simulations, we did not encounter a single iteration in which $2 \cdot \frac{\norm{\mathcal R_H (x)}_1}{d} \le \frac{\norm{x}_2^2}{\norm{\mathcal R_H(x)}_1}$. 

Figure \ref{fig:ef:drive} shows that with and without EF, \alg and \algp result in similar performance. Also, as evident in Figure~\ref{fig:ef:others}, \alg and \algp show favorable performance in comparison to EF-Sign SGD and Top-K.
Here, we configured the Top-K algorithm with $K=d/32$; note that this requires more than one bit per coordinate, as the sender needs to encode the indices of the sent coordinates in addition to their values.

\subsection{Additional vNMSE-Speed Tradeoff Details and Results}\label{app:additional_speed_results}

\paragraph{\textbf{Experiment Configuration.}}
The experiments were executed over Intel Core i9-10980XE CPU (18 cores, 3.00 GHz, and 24.75 MB cache), 128 GB RAM, NVIDIA GeForce RTX 3090 GPU, Ubuntu 20.04.2 LTS operating system, and CUDA release 11.1, V11.1.105.
We draw $100$ vectors\footnote{except for $d=2^{25}$, where we use 10 encodings of each vector due to the fact that in such high dimension, random vectors are concentrated close to their expectation.} from a Lognormal(0,1) distribution and measure $100$ encodings of each vector. Similarly to Figure~\ref{fig:theory_is_cool}, each of $n=10$ clients are given the same vector at each iteration.
Due to its high runtime, we only evaluate the uniform random rotation on the low ($d\in\set{128,8192}$) dimensions.

\paragraph{\textbf{Experiment Results Summary.}}
The results, given in Table~\ref{tbl:weAreFast}, show that \alg and \algp are the most accurate algorithms and their NMSE converges to the theoretical NMSE of $\frac{\pi/2-1}{10}\approx 0.0571$ even in dimensions as small as $2^{13}$. 
In even lower dimensions, such as $d=128$, \algp is more accurate than \alg, albeit slightly slower. Running \alg and \algp with uniform rotation instead of Hadamard yield even lower NMSE, but also takes more time.
\alg is also almost as fast as Hadamard with 1-bit SQ~\cite{pmlr-v70-suresh17a} and is 12x faster than Kashin with 1-bit SQ.
We conclude that \alg with Hadamard offers the best tradeoff in all cases, except when the dimensions is small, in which case one may choose to use \algp and possibly uniform random rotation.

\begin{table}[H]
\resizebox{\textwidth}{!}{%
\begin{tabular}{r|l|l|l|l|l|l|}
\cline{2-7}
\multicolumn{1}{l|}{Dimension ($d$)} & \multicolumn{1}{c|}{\begin{tabular}[c]{@{}c@{}}Hadamard\\ + 1-bit SQ\end{tabular}} & \multicolumn{1}{c|}{\begin{tabular}[c]{@{}c@{}}Kashin\\ + 1-bit SQ\end{tabular}} & \multicolumn{1}{c|}{\begin{tabular}[c]{@{}c@{}}Drive \\ (Uniform)\end{tabular}} & \multicolumn{1}{c|}{\begin{tabular}[c]{@{}c@{}}Drive$^+$\\ (Uniform)\end{tabular}} & \multicolumn{1}{c|}{\begin{tabular}[c]{@{}c@{}}Drive\\  (Hadamard)\end{tabular}} & \multicolumn{1}{c|}{\begin{tabular}[c]{@{}c@{}}Drive$^+$ \\ (Hadamard)\end{tabular}} \\ \hline
\multicolumn{1}{|r|}{128}            & 0.5308,        {\textit{0.93}}  & 0.2550,  \textit{6.06}   &         0.0567,  \textit{14.95} & \textbf{0.0547}, \textit{18.25}     & 0.0591,                  \textit{0.98}   &         0.0591, \textit{2.07}     \\ \hline
\multicolumn{1}{|r|}{8,192}          & 1.3338,        {\textit{1.53}}  & 0.3180,  \textit{10.9}   & \textbf{0.0571}, \textit{2646}  & \textbf{0.0571}, \textit{2672}     & \textbf{0.0571},        {\textit{1.58}}  & \textbf{0.0571}, \textit{2.93}     \\ \hline
\multicolumn{1}{|r|}{524,288}        & 2.1456,        {\textit{3.76}}  & 0.3178,  \textit{30.8}   & ---                             & ---                                & \textbf{0.0571},         \textit{3.78}   & \textbf{0.0571}, \textit{5.72}     \\ \hline
\multicolumn{1}{|r|}{33,554,432}     & 2.9332,         \textit{40.9}   & 0.3179,  \textit{1714}   & ---                             & ---                                & \textbf{0.0571},        {\textit{43.0}}  & \textbf{0.0571}, \textit{252}      \\ \hline
\end{tabular}%
}
\caption{A \emph{commodity machine} comparison of empirical NMSE and average per-vector encoding time (in milliseconds) for distributed mean estimation with $n=10$ clients (same as in Figure~\ref{fig:theory_is_cool}) and Lognormal(0,1) \mbox{distribution. Each entry is a (NMSE, \textit{time}) tuple and the best result is highlighted in~\textbf{bold}.}}
\label{table:we_are_fast_on_commodity_machine}
\end{table}

\paragraph{\textbf{Speed Measurements on a Commodity Machine.}}
The experiment in Table~\ref{tbl:weAreFast} in the paper was performed on a high-end server with a NVIDIA GeForce RTX 3090 GPU. For completeness, we also run measurements on a commodity machine.
Specifically, we re-execute the experiments over Intel Core i7-7700 CPU (8 cores, 3.60 GHz, and 8 MB cache), 64 GB RAM, NVIDIA GeForce GTX 1060 (6GB) GPU, Windows 10 (build 18363.1556) operating system and CUDA release 10.2, V10.2.89.
The results, shown in Table~\ref{table:we_are_fast_on_commodity_machine} show a similar trend.   
Again, as expected, \alg and \algp with uniform rotation are more accurate for small dimensions but are significantly slower.
On this machine, \alg is 6.18$\times$-39.8$\times$ faster than Kashin and, again, about as fast as Hadamard.
All NMSE measurements, as expected yielded the same results as in Table~\ref{tbl:weAreFast}.

\subsection{Federated Learning Experiments Configuration}\label{app:fl_details}

The experiments were executed over a university cluster that contains several types of NVIDIA GPUs (TITAN X (Pascal), TITAN Xp, GeForce GTX 1080 Ti, GeForce RTX 2080 Ti, and GeForce RTX 3090) with no specific hardware guarantee.

\begin{table*}[h]
\centering
\begin{tabular}{|l||c|c|c|c|c|}
\hline
\textbf{Task} & \textbf{Clients per round} & \textbf{Rounds} & \textbf{Batch size} & \textbf{Client lr} & \textbf{Server lr} \\ \hline \hline
EMNIST & 10   & 1500  & 20 & 0.1  & 1 \\ \hline
CIFAR-100 & 10 & 6000    & 20      & 0.1 & 0.1 \\ \hline
Shakespeare & 10 & 1200 & 4 & 1  &  1\\ \hline
Stack Overflow & 50 & 1500 & 16      &  0.3 &  1\\ \hline
\end{tabular}
\caption{Summary of main hyperparameters for federated learning tasks.}
\label{table:fl_params} 
\end{table*}

Table~\ref{table:fl_params} contains a summary of the hyperparameters we used. Those generally correspond to the best setup described in \cite[Appendix~D]{reddi2020adaptive} for FedAvg except for the following changes: (1) We increased the number of rounds of the EMNIST task from 4000 to 6000 to show a clear convergence; (2) We used a server momentum of 0.9 (named \emph{FedAvgM} in \cite{reddi2020adaptive}) for the Stack Overflow task since otherwise, this task produces noisy results due to its extreme unbalancedness in the number of samples per client; and (3) For FetchSGD on EMNIST and Stack Overflow we set the server learning rate to 0.3 as it offered improved performance.

We compress each layer separably due to the implementation of the libraries we use and algorithms we compare to (specifically, see ``\textsc{class~EncodedSumFactory}''~\cite{tensorflowfedencsum} in TensorFlow Federated). This layer-wise compression reduces the dimension in practice and explains the difference seen between \alg and \algp performance. Additionally, all compression algorithms skip layers with less than 10,000 parameters since those are usually either normalization or last fully-connected layers, \mbox{which are more cost-effective to send as-is.}

\subsection{Distributed Learning Experiments Configuration}\label{app:dl_details}
The experiments were executed over Intel Core i9-10980XE CPU (18 cores, 3.00 GHz, and 24.75 MB cache), 128 GB RAM, NVIDIA GeForce RTX 3090 GPU, and Ubuntu 20.04.2 LTS OS.

For the distributed CNN training we have used an SGD optimizer with a Cross entropy loss criterion, a momentum of $0.9$, and a weight decay of 5e-4. The learning rate of all algorithms was set to $0.1$ for the CIFAR-10 over ResNet-9 experiment and $0.05$ for the CIFAR-100 over ResNet-18 experiments. TernGrad was an exception and its learning rate was set to $0.01$ and $0.005$ as it offered improved performance. The batch size in all experiments and clients was set to 128.




\subsection{\new{K-Means and Power Iteration Simulation Results}}
\label{subsection:power_iteration_appendix}

\begin{figure}[]
\ifdefined\arxiv  
  \vspace*{-4mm}
\fi
\centering
\centerline{\includegraphics[width=\textwidth, trim=0 70 0 0, clip]{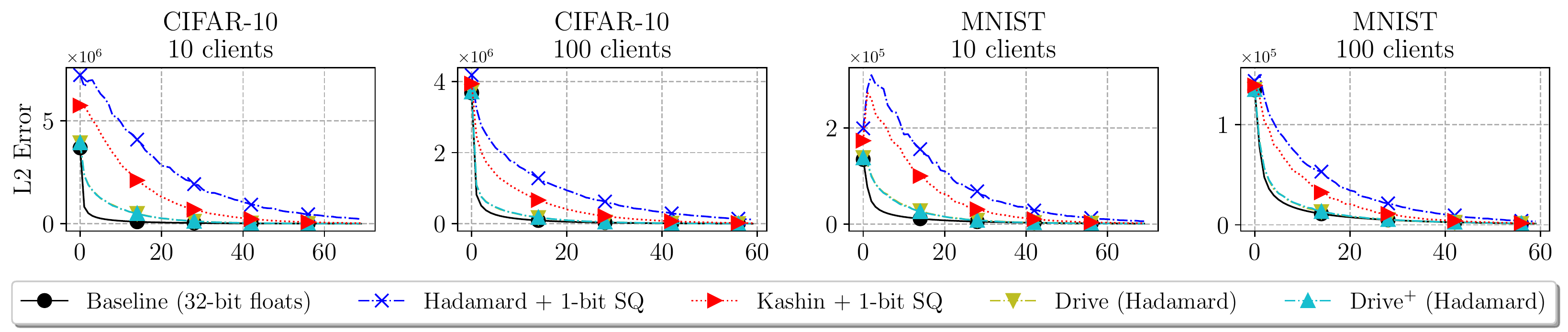}}
\ifdefined\arxiv  
  \vspace*{-2mm}
\fi
\caption{
 \mbox{L2 error per round for distributed power iteration.}
}
\ifdefined\arxiv  
  \vspace*{-2mm}
\fi
\label{fig:power_iteration}
\end{figure}

\begin{figure}[]
\centering
\centerline{\includegraphics[width=\textwidth, trim=0 70 0 0, clip]{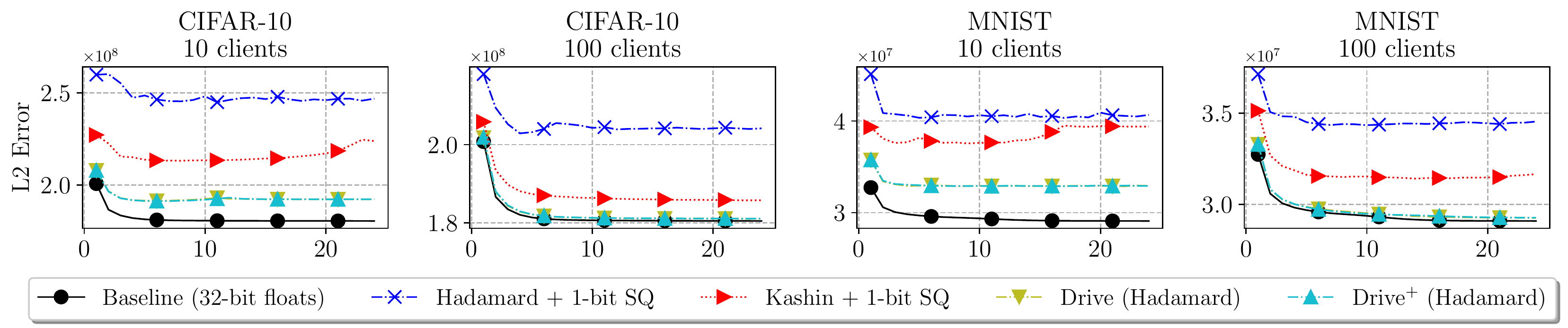}}
\ifdefined\arxiv  
\vspace*{-2mm}
\fi
\caption{
  {L2 error per round for distributed K-Means.}
\ifdefined\arxiv  
  \vspace*{-6mm}
\fi
}
\label{fig:kmeans}
\end{figure}

\ifdefined\arxiv
\begin{figure}[b]
\centering
\centerline{\includegraphics[width=\textwidth, trim=0 60 0 0, clip]{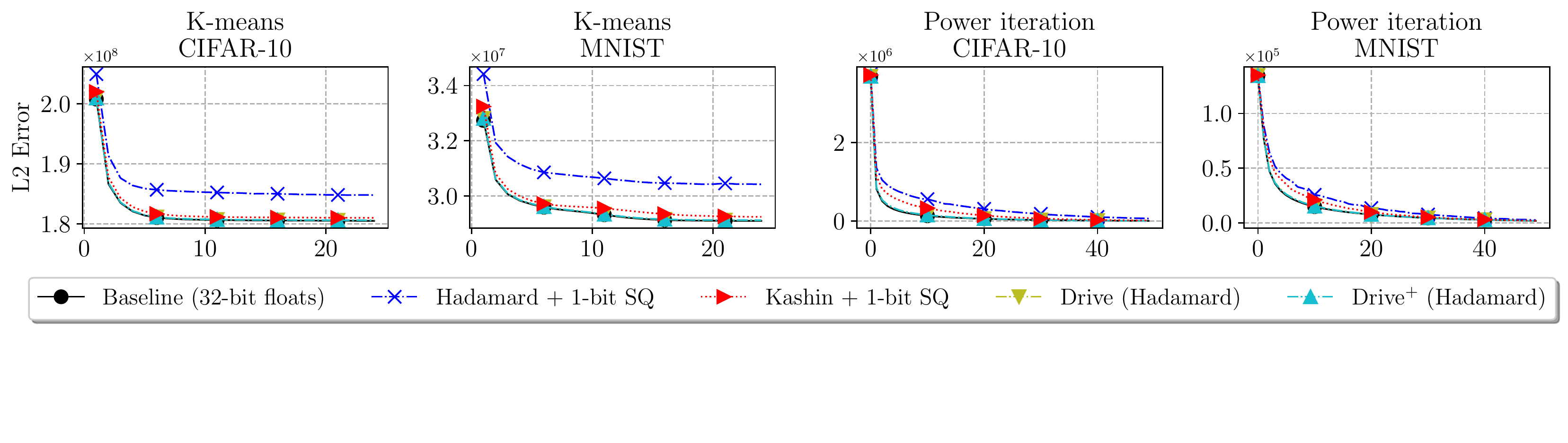}}
\caption{
Distributed K-Means and power iteration of CIFAR-10 and MNIST with 1000 clients.}
\label{fig:not-dnn-1000-clients}
\end{figure}
\else
\begin{figure}[]
\centering
\centerline{\includegraphics[width=\textwidth, trim=0 60 0 0, clip]{figures/not-dnn-appendix-1000-clients.pdf}}
\caption{
Distributed K-Means and power iteration of CIFAR-10 and MNIST with 1000 clients.}
\label{fig:not-dnn-1000-clients}
\end{figure}
\fi

\paragraph{\textbf{{Distributed Power Iteration}.\quad}} 
Power Iteration is a simple method for approximating the dominant Eigenvalues and Eigenvectors of a matrix. It is often used as a sub-routine in more complicated tasks such as Principal Component Analysis. In our distributed setting, at each epoch, the server broadcasts the current estimated top eigenvector to all clients. Then, each client: (1) updates the top eigenvector based on its local data; (2) computes its diff to the current estimated top eigenvector; (3) compresses the diff vector; (4) sends it to the server, possibly scaled by a \emph{learning rate} to ensure convergence to due the high estimation variance. 
%
Figure~\ref{fig:power_iteration} presents the results in a setting with $n=10$ and $n=100$ clients and a learning rate of 0.1. 

In both datasets, and for both $n=10$ and $n=100$ clients, \alg and \algp have lower error and faster convergence than other compression algorithms. Our convergence is faster when using a larger number of clients, which matches the theoretical analysis which shows that the estimates of \alg (Hadamard) and \algp (Hadamard) are (nearly) unbiased for such a high dimension.

\ifspacinglines
\fi

\paragraph{\textbf{{Distributed K-Means}.\quad}} 
In this distributed learning task, the server coordinates a run of Lloyd's algorithm over a distributed dataset. At each training round, the server broadcasts the centroids' coordinates to the clients. Then, each client: (1) updates these coordinates based on its local data; (2) compresses them; (3) sends them to the server. The centroid weights (i.e., the number of observations assigned to each centroid) require fewer bits and are essential for K-Means to converge quickly; therefore, the weights are sent without compression.
%

Figure~\ref{fig:kmeans} presents the results with $n=10$ and $n=100$ clients. 
As shown, \alg and \algp have a lower error than other compression algorithms. While none of the compression methods is competitive with the baseline when using $n=10$ clients, \alg and \algp are the only algorithms that have comparable accuracy for $n=100$. 


\paragraph{\textbf{{Increasing the Number of Clients}.\quad}}  Figure \ref{fig:not-dnn-1000-clients} depicts simulation results for the distributed K-Means and distributed Power Iteration experiments with $n=1000$ clients.  
The results indicate a similar trend. \alg and \algp offer the lowest L2 error among the low-bandwidth techniques, and they are now closer to the uncompressed baseline due to the larger number of clients, which further reduces the estimation variance.

\subsection{\new{Comparison With Entropy Encoding}}\label{subsec:app:ee}

We conduct an experiment for distributed mean estimation with $n=10$ clients (same as in Figure \ref{fig:theory_is_cool} and Table \ref{tbl:weAreFast}) with the Lognormal(0,1)
distribution. We determine how many bits per coordinate are required by a compression method that uses stochastic quantization followed by entropy encoding to reach the same NMSE as that of \alg (that uses a single bit per coordinate). We examine two methods: (1) standard stochastic quantization followed by Huffman encoding; (2) stochastic quantization with levels chosen for entropy encoding followed by Huffman encoding, as proposed in \cite{pmlr-v70-suresh17a} (see Section 4). We refer to this second variant as Enhanced SQ.  We introduced an increasing number of quantization levels for each of these methods until the resulting empirical NMSE was similar to that of \alg. We repeat each experiment 100 times and report the mean value. The results summarized in Table \ref{tab:entropy_encoding_cmp} indicate that these methods require at least 1.261 bits per coordinate. 

\begin{table}[t]
\centering
\begin{tabular}{r|c|c|c|}
\cline{2-4}
\multicolumn{1}{c|}{\multirow{2}{*}{Dimension ($d$)}} & \multirow{2}{*}{\begin{tabular}[c]{@{}c@{}}\alg (Hadamard)\\ NMSE (1-bit)\end{tabular}} & \multicolumn{2}{c|}{Required bits/coordinates} \\ \cline{3-4} 
\multicolumn{1}{c|}{}                                 &                                                                                                                      & SQ + Huffman      & Enhanced SQ + Huffman~\cite{pmlr-v70-suresh17a}      \\ \hline
\multicolumn{1}{|r|}{128}                             & 0.0591                                                                                                               & 1.284             & 1.261                      \\ \hline
\multicolumn{1}{|r|}{8,192}                           & 0.0571                                                                                                               & 1.335             & 1.295                      \\ \hline
\multicolumn{1}{|r|}{524,288}                         & 0.0571                                                                                                               & 1.324             & 1.298                      \\ \hline
\multicolumn{1}{|r|}{33,554,432}                      & 0.0571                                                                                                               & 1.321              & 1.301                      \\ \hline
\end{tabular}
\vspace{1mm}
\caption{Empirical NMSE for distributed mean estimation with $n=10$ clients (same as in Figure \ref{fig:theory_is_cool} and Table \ref{tbl:weAreFast}) with the Lognormal(0,1) distribution. We find the number of bits per coordinate that are required by the entropy encoding methods to reach the same empirical NMSE as that of \alg (that uses a single bit per coordinate).}
\label{tab:entropy_encoding_cmp}
\end{table}

There are other uses of entropy encoding techniques for machine learning tasks as suggested by \cite{pmlr-v70-suresh17a,NIPS2017_6c340f25}. Such methods often introduce higher computational costs.  We leave further comparisons of \alg and entropy encoding for specific problems for future work.  

\ifdefined\arxiv
\vbox{
\fi
\section{Lower Bounds}\label{app:lower_bounds}
It is proven in~\cite{iaab006} that \emph{any} unbiased algorithm must have a 1b~-~Vector Estimation vNMSE of at least $\frac{1}{3}$ and any biased algorithm incurs a vNMSE of at least $\frac{1}{4}$. 
However, it does not appear that their proof accounts for algorithms in which \sender and \receiver may have shared randomness, as is the case for all algorithms proposed in our paper.
The reason is that in~\cite{iaab006} it is assumed that if \sender sends $b$ bits, \receiver only has $2^b$ possible estimates, one for each of the $2^b$ possible messages. However, if shared randomness is allowed, there are more possible values for $\widehat x$. In fact, if the number of shared random bits is not restricted, the number of different possible estimates can be unbounded. 
Additionally, it is shown in~\cite{ben2020send} that for a single number ($d=1$), using shared randomness allows algorithms to have an MSE lower than the optimal algorithm when no such randomness is allowed.

{Nonetheless, their result implies that any algorithm that uses $O(d)$ shared random bits (as does our Hadamard-based variants) has a vNMSE of $\Omega(1)$, i.e., \alg and \algp are asymptotically optimal. 
The reason is that \sender could generate (private) random bits and send them as part of the message to mimic shared random bits. For example, in our Hadamard variant, \sender could send the random diagonal matrix $D$ using $d$ additional bits. This implies that any biased algorithm with at most $d$ shared random bits must have a vNMSE of at least $\frac{1}{16}$ and any unbiased \mbox{solution must have a vNMSE of at least $\frac{1}{15}$.}}
\ifdefined\arxiv
}
\fi





\end{document}


\section{ideas}
\begin{enumerate}
    \item Subtract the minimum coordinate to avoid encoding the sign bit (for additive compression only).
    \item Augmented sketch to reduce the number of bits and improve accuracy?
    \item Previous top-k indices.
    \item Two phase computation.
    \item Remove (-0) for the additive compression.
\end{enumerate}

\section{even more ideas}
\begin{enumerate}
    \item use biased BMV + error sketch.
    \item chose indices with shared randomness
    \item solve the sign majority vote directly. May work well with sampling.
    \item renaming: if $max(q0,q1,q2) >0.25$ run ours, else run RR with renaming (per layer??).
\end{enumerate}

\section{and more ideas}
\begin{enumerate}
    \item random sparsification + instead of using a constant chance per coordinate, sample coordinates in proportion to their distance from last gradient (clients run 2 batches). This approach, I think, implements randomized diffGrad \citep{8939562} in the client side. \\
    TODO related? \citep{mishchenko2019distributed}
    \item Optimize quantization levels using distribution estimation. May want to look at \citep{pmlr-v119-fu20c}. 
\end{enumerate}

\section{competition (few-bit distributed learning)}

\begin{itemize}

    \item \textbf{1BitSGD} variants
    
    \begin{itemize}

    \item \citep{seide20141} \\ \emph{1-bit stochastic gradient descent and its application to data-parallel distributed training of speech DNNs} 
    
   \begin{enumerate}
        \item \textbf{Error feedback}
        \item Split into buckets (matrix columns in the scope of the paper)
        \item Scale using the means of the values mapped to 0 and 1
        \item Simple 1-bit sign
    \end{enumerate}
        
    \item \citep{strom2015scalable} \\ \emph{Scalable distributed DNN training using commodity GPU cloud computing} 
    
    \begin{enumerate}
        \item \textbf{Error feedback}
        \item Split into buckets (matrix columns in the scope of the paper)
        \item Sparse 1-bit updates: Given a constant threshold $\tau$, uses 31 bits for sending the index $i$ for each entry $|g[i]| > \tau$ and uses the last 32 bit for the sign
        \item Suggests following with delta encoding and then Golomb-Rice encoding but does not present experiments
    \end{enumerate}
    
    In his other work Str\"om he claimed that params are often close to zero \citep{659000, strom1997sparse}. This is something we may want to cite.   
    
    \item \citep{7835789} \\ \emph{Communication quantization for data-parallel  training  of  deep  neural  networks} 
    
    Adds ``adaptivity'' to Str\"om's work: instead of a constant $\tau$ they use a fixed \emph{proportion} $\pi$, and send only the biggest $\%\pi$ positive elements and $\%\pi$ of the smallest negative elements.
    
    The server scales the values using means of positive and negatives values separably (same as 1-bit quantization).
    
    \textbf{(Uses error feedback)}
    
    \item TODO \citep{8852172} \\ 
    \emph{Sparse Binary Compression: Towards Distributed Deep Learning with minimal Communication}
    
    \item TODO \citep{NEURIPS2018_33b3214d} \\
    \emph{ATOMO: Communication-efficient Learning via Atomic Sparsification}
    
    \end{itemize}
    
    \item \citep{bernstein2018signsgd, bernstein2019signsgd} \\
    \emph{signSGD: Compressed Optimisation for Non-Convex Problems} \\
    \emph{signSGD with Majority Vote is Communication Efficient and Fault Tolerant}
    
    Does less than 1bitSGD, proves more: \\
    In the ``Distributed training by majority vote" variant, using an aggregated gradient calculated as $sign(\sum_m{sign(\tilde{g}_m)})$), can converges \emph{without error-feedback or scaling}.
    
    \item \textbf{Stochastic rounding gen 1}
    
    \begin{itemize}
        \item \citep{NIPS2017_6c340f25} \\ 
        \emph{QSGD: Communication-Efficient SGD via Gradient Quantization and Encoding}
        
        \begin{enumerate}
            \item Split to buckets
            \item Scale using $L^2$-norm or $L^\infty$-norm per bucket
            \item Stochastic rounding into $s$ levels
            \item \emph{Elias omega  coding}
        \end{enumerate}
        
        \item \citep{wen2017terngrad} \\ \emph{Terngrad: Ternary gradients to reduce communication in distributed deep learning}
        
        
        \begin{enumerate}
            \item Layer-wise
            \item Clip gradients using a magical $2.5\sigma$ cut
            \item Scale using each layer's $L^\infty$-norm \\
            Also suggests adding a round of communication to get the maximum $L^\infty$-norm among all clients
            \item Stochastic rounding into $\{-1, 0, 1\}$
        \end{enumerate}
        
        \item \citep{pmlr-v70-suresh17a} \\
        \emph{Distributed Mean Estimation with Limited Communication}
        
        QSGD with random rotations and variable-length coding
        
        TODO re-read with focus on the following mention: ``\cite{konevcny2018randomized} Showed that $\pi_{sb}$ can be improved further by optimizing the choice of stochastic quantization boundaries''
        
        \item TODO DIANA \emph{Distributed Learning with Compressed Gradient Differences}
    \end{itemize}
    
    \item \textbf{Stochastic rounding gen 2}
    \begin{enumerate}
        \item \citep{ramezani2019nuqsgd} (not published) \\
        \emph{NUQSGD: Improved communication efficiency for data-parallel sgd via nonuniform quantization}
        \item \citep{horvath2019natural} (not published) \\
        \emph{Natural compression for distributed deep learning}
        \item \citep{pmlr-v108-mayekar20a} \\
        \emph{RATQ: A Universal Fixed-Length Quantizer for Stochastic Optimization}
        
        Seems to be random rotations (Walsh-Hadamard Matrix) + stochastic rounding on buckets + with ``adaptive tetration-growth?" scaling (exponential jumps in scaling size, or something like that...)
        
        \item \citep{shlezinger2020uveqfed} \\
        \emph{UVeQFed: Universal Vector Quantization for Federated Learning}
        
        Substractive dithering among other things
        
        \item \citep{chen2020distributed} \\
        \emph{Distributed training with heterogeneous data: Bridging median-and mean-based algorithms}
        Describes signSGD with some added unimodel noise, not exactly stochastic rounding but possibly related
        
        \item \citep{jin2020stochasticsign} \\
        \emph{Stochastic-Sign SGD for Federated Learning with Theoretical Guarantees}
        
        signSGD with stochastic rounding (which the authors don't seem to acknowledge) - may find something useful in the analysis

    \end{enumerate}
    
    \item \textbf{Stochastic rounding gen 3}
    
    \begin{enumerate}
        \item \citep{caldas2018expanding}, \citep{chen2020breaking} \\
        TODO \emph{Kasinh's}
    \end{enumerate}
    
    \item TODO \citep{karimireddy2019error} \\
    \emph{Error Feedback Fixes SignSGD and other Gradient Compression Schemes}
        
    \item sparsity
    
        TODO
        
        \begin{itemize}
            \item \citep{aji2017sparse} \\ \emph{Sparse Communication for Distributed Gradient Descent} 
            
            Drops $\%R$ smallest in absolute value
            
            \item \citep{konevcny2018randomized} \\
            Randomized distributed mean estimation: Accuracy vs. communication
            
            \item \citep{NEURIPS2018_3328bdf9} \\
            \emph{Gradient Sparsification for Communication-Efficient Distributed Optimization}
        \end{itemize}
    
    \item Works with a client state (we just want to mention why we don't compare against them) (since we'll lose if we do):
    
    \begin{itemize}
        \item \emph{Deep Gradient Compression: Reducing the Communication Bandwidth for Distributed Training} \citep{lin2018deep}
    \end{itemize}
    
\end{itemize}
\section{TODO}

\subsection{Introduction (TBD)}

\paragraph{Background}
Machine learning has led to breakthroughs in various fields, such as natural language processing and computer vision. Much of this success hinges on vast amounts of data, but collecting this data from individuals or organizations can sometimes be undesirable due to data ownership and locality concerns. These concerns apply to many applications across many industries, including health care, telecommunications, financial services, IoT, defense, and insurance. \emph{Federated Learning} (FL) \citep{konevcny2015federated, mcmahan2017communication, kairouz2019advances, bonawitz2019towards} is a distributed machine learning paradigm where training data resides
at autonomous client machines and the learning process is facilitated by a central server. The server maintains a shared model while the clients send improvements to the model. While the clients have plenty of memory and processing time, communication is done via the datacenter's network and bandwidth is a limited resource. Therefore, we experience a trade-off between learning efficiency and consumed bandwidth. Consequently, many works optimize the learning quality to bandwidth ratio\cite{seide20141,strom1997sparse}. Our work has similar goals to these existing works, but revisits the problem from first principles, and derives rigours analysis as well as a reduced memory consumption. Plainly speaking, we encode the transferred information better than the alternatives, and we dynamically adjust the encoding as the learning proceeds to guarantee that the loss from our method is negligible compared to other inaccuracies in the learning process.

\subsection{Preliminaries (TBD)}

\textbf{Optimization Goal} We are given $K$ clients where each client $k$ has a local collection $Z_k$ of $n_k$ samples taken IID\ from some unknown distribution over sample space $\bm{Z}$. Our objective is \emph{global} empirical risk minimization (ERM) for some loss function class $ \ell(w; \cdot)\colon \bm{Z} \to \mathbb{R}$, parameterized by $w \in \mathbb{R}^d$ \footnote{We note that some previous FL works specify a more generic finite-sum objective \cite{mcmahan2017communication}. However, this work investigates client-declared sample sizes, whose meaning is clear under the ERM interpretation but seems meaningless in  the finite-sum objective setting.}:

\begin{gather} \label{eq:1}
 \min_{w \in \mathbb{R}^d} F(w),  \\
 \text{where} \nonumber \\
 Z=\bigcup_{k \in [K]} Z_k;\ n = |Z|;\ F(w) \coloneqq \frac{1}{n} \sum_{z \in Z} \ell(w ; z).  \nonumber
\end{gather}

\paragraph{Collaboration Model} We restrict ourselves to the FL paradigm, which leaves the training data distributed on client machines, and learns a shared model by iterating between client updates and server aggregation.
\\~

\noindent
\textbf{Assumption 1 (Smoothness).} A function $f: \mathbbm{R}^d \to \mathbbm{R}$ is L-smooth if it holds that 
\begin{align*}
 \norm{\nabla f(x) - \nabla f(y)} \le L \norm{x - y} \quad \forall x,y \in dom(f).   
\end{align*}
%
\textbf{Assumption 2 (Gradient).} The query for a stochastic gradient returns $g$ such that 
\begin{align*}
\mathbbm{E}[g]= \nabla f(x), \quad \mathbbm{E}[\norm{g}^2] \le \sigma^2 \quad\forall x\in dom(f).    
\end{align*}
%
\textbf{Assumption 3 (Sketching).} Each client may sketch its gradient. We require the following properties from the sketching procedure: denote by $\mathcal{S}$ and $\mathcal{U}$ as the sketching and unsketching operators. Then
\begin{itemize}
    \item $\mathbbm{E} [\unsketch{\sketch{x}}] = x$.
    
    \item $\mathbbm{E} [\norm{\unsketch{\sketch{x}}}^2] \le k_s \norm{x}^2$, \quad $k_s \in [1,\infty)$.
    
    \item $\sketch{\sum_i c_i x_i} = \sum_i c_i \sketch{x_i}$.
\end{itemize}
 These first two conditions also lead to  These first two conditions also lead to 
 \begin{align*}
  \mathbbm{E} \big[\norm{x-\frac{1}{k_s}\unsketch{\sketch{x}}}^2\big] \le (1-\frac{1}{k_s})\norm{x}^2.  
 \end{align*}
%
\textbf{Assumption 4 (Compression).} Each client may apply any other compression technique, either on the gradient itself or on the sketch. We require the following properties from the compression procedure: denote by $\mathcal{C}$ and $\mathcal{D}$ as the compression and decompression operators. Let $i$ and $i'$ be two client ID's and $t$ and $t'$ be two rounds. Then, 
\begin{itemize}

    \item $\mathbbm{E} [\decompress{\compress{x, i, t}}] = x$, \quad $\forall (x,i,t)$.

    \item $\mathbbm{E} [\norm{\decompress{\compress{x, i, t}}}]^2 \le k_c \norm{x}^2$, \quad $\forall (x,i,t)$, $k_c \in [1,\infty)$.
    
    \item If $i \neq i'$ or $t \neq t'$ then $\decompress{\compress{x, i, t}}$ and $\decompress{\compress{y, i', t'}}$ are uncorrelated random $\forall x,y$.
    
\end{itemize}

